\documentclass[preprint]{elsarticle}

\usepackage[utf8]{inputenc} 
\usepackage[T1]{fontenc}    
\usepackage{hyperref}       
\usepackage{url}            
\usepackage{booktabs}       
\usepackage{amsfonts}       
\usepackage{nicefrac}       
\usepackage{microtype}      
\usepackage{amsthm}
\usepackage{alg}
\usepackage{algorithmic}
\usepackage{paralist,subcaption}
\usepackage{mathtools}
\usepackage{fancyvrb}
\usepackage{pgfplots}
\pgfplotsset{compat=1.16}
\newtheorem{theorem}{Theorem}
\newtheorem{lemma}[theorem]{Lemma}

\newtheorem{definition}{Definition}
\newtheorem{requirement}{Requirement}

\usetikzlibrary{arrows, quotes,backgrounds,calc}
\usepackage{tikz}
\usetikzlibrary{arrows.meta}
\usetikzlibrary{arrows,shapes,decorations,automata,backgrounds,petri}
\tikzset{point/.style = {fill=black,circle,inner sep=0.7pt}}
\tikzset{
  graph vertex/.style={
    circle,
    draw,
  },
    graph vertexx/.style={
    draw,
  },
  graph directed edge/.style={
    ->,
    >=stealth,
    thick,
  },
  graph tree edge/.style={
    graph directed edge
  },
  graph forward edge/.style={
    graph directed edge,
    every edge/.style={
      edge node={node [fill=white,font=\scriptsize] {f}},
      loosely dotted,
      draw,
    },
  },
  graph back edge/.style={
    graph directed edge,
    every edge/.style={
      edge node={node [fill=white,font=\scriptsize] {b}},
      densely dotted,
      draw,
    },
  },
  graph cross edge/.style={
    graph directed edge,
    every edge/.style={
      edge node={node [fill=white,font=\scriptsize] {c}},
      dotted,
      draw,
    },
  },
}

\usepackage{color,amssymb}

\title{Minimal Conditions for Beneficial Local Search}

\author{%
  Mark G. Wallace\\
  Dept. of Data Science and Artificial Intelligence\\
  Faculty of Computer Science\\
  Monash University\\
  Clayton, Vic 3168\\
  Australia\\
  \texttt{mark.wallace@monash.edu} \\
}

\begin{document}
\begin{abstract}
This paper investigates why it is beneficial, when solving a problem, to search in the neighbourhood of a current solution.
There is a huge range of algorithms that use this for local improvement, generally within a larger meta-heuristic framework, for applications in machine learning and optimisation in general.
The paper identifies properties of problems and neighbourhoods that support two novel proofs that neighbourhood search is beneficial over blind search. 
These are: firstly a proof that search within the neighbourhood is more likely to find an improving solution in a single search step than blind search; and secondly a proof that a local improvement, using a sequence of neighbourhood search steps, is likely to achieve a greater improvement than a sequence of blind search steps.
To explore the practical impact of these properties, a range of problem classes and neighbourhoods are generated, where these properties are satisfied to different degrees.  Experiments reveal that the benefits of neighbourhood search vary dramatically in consequence.
Examples of classical combinatorial optimisation problems are analysed, in order to demonstrate that the underlying theory is reflected in practice.
\end{abstract}


\maketitle

\section{Introduction}
\subsection{Neighbourhood search and local descent}
There is a wide variety of techniques for tackling large scale combinatorial optimisation problems.  
Incomplete search methods are typically used to achieve the required scalability.
Many of these methods involve a sub-algorithm where a current solution, or set of solutions, are modified in some way to produce new candidate solutions.  We call this neighbourhood search.
Naturally these methods build a smart and subtle superstructure upon neighbourhood search, to achieve a balance of diversification and intensification, and fast improvement in the fitness of the best solution found so far.
This paper examines the benefit of neighbourhood search and local descent (using a sequence of neighbourhood search steps), and therefore why these methods deploy them.

Blind search selects candidate solutions at random, evaluating the fitness of each one.  
In this paper we present minimal general conditions under which a neighbourhood search algorithm can perform better than blind search.
We admit the worst case, where
the neighbourhood search algorithm simply picks candidate solutions blindly from the neighbourhood of the current solution.
We say the neighbourhood search performs better than blind search if the probability the next candidate solution has better fitness than the current solution is higher than the probability a better solution is selected by blind search.

The "no free lunch" theorems \cite{nofreelunch} hold for an optimisation algorithm applied to an unknown fitness function from a given class of functions.  The value of the function is revealed for a given candidate solution only when that candidate is selected for evaluation.
This is termed {\em black box} optimisation.
While the proofs in this paper are presented for black box optimisation, the general conditions are shown also to apply to {\em white box} optimisation, where the fitness function is given, but is challenging to optimise.

The term \emph{local descent}  in this paper means searching in the neighbourhood until a point is found with equal or better fitness, then making this the current point and searching from there.
Thus local descent is an extension of neighbourhood search, also called hill climbing.

When neighbours have very similar fitness, even if neighbourhood search yields a better candidate, it may only be a little better.  On the other hand, blind search might occasionally find \emph{much} better candidates.  Thus, even if neighbourhood search had higher probability of improvement, the expected amount of improvement with blind search could be greater than with neighbourhood search. 
Indeed this is illustrated with an example in section \ref{sec:betterblind}.

However using local descent,
when a better candidate has been found it becomes the current candidate, the potential temporary advantage of blind search may eventually be overcome. 

This leads us to ask \emph{under what conditions is the expected rate of improvement achieved by local descent  better than blind search}?
Specifically, we investigate
the expected number of trials - or points it evaluates - 
for local descent to reach improved fitness levels.

\subsection{Neighbours Similar Fitness}
The first general condition under which neighbourhood search is beneficial is the following condition on the neighbourhoods.
Neighbour's Similar Fitness (NSF) is a property of neighbourhoods made precise in sections \ref{sec:nsf} and \ref{sec:nsf2}.
Essentially, a candidate solution chosen from the neighbourhood of the current solution is more likely to have a fitness similar to the current solution than a candidate solution chosen from the search space as a whole.

We claim that the neighbourhood operators used in solving many combinatorial problems yield neighbourhoods with the NSF property.
Indeed, the objective function in many combinatorial problems is a sum of terms - where the number of terms in the sum grows with the size of the problem instance.
Often neighbouring solutions only differ in a few of these terms.
For example in MaxSAT problems, the objective is the sum of constraint violations, 
and in the travelling salesman problem ({\it TSP}) the objective is the sum of the distances between each pair of successive locations on the tour.
For MaxSAT problems a neighbour is found by flipping the value of a variable; and for the travelling salesman problem, by a 2-opt or a 3-opt swap: in each case this ensures that the neighbour has a similar fitness because most of the terms in the fitness function are unchanged.
This paper shows that NSF is key to escaping from the conditions of the no free lunch theorems.

Later sections \ref{sec:2sat} and \ref{sec:tspexperiments} will demonstrate experimentally that the flip neighbourhood operator for 2-SAT, and the 2-swap neighbourhood operator for the TSP, both have the NSF property.

\subsection{Decreasing probability of fitness near the optimum}

The second, and final, general condition under which neighbourhood search is beneficial is a condition on the current solution, and in particular on the fitness of the current solution.

The {\em fitness probability} is the probability that an arbitrary point in the search space of a problem has a given fitness.
The condition on the fitness $f$ of the current solution, is that the fitness probabilities of fitnesses around $f$ are monotonically decreasing.  Specifically if the difference between $f$ and the optimum fitness is $\delta$, then the fitness probabilities are monotonically decreasing in the interval $f-\delta \ldots f+\delta$.

This property always holds of the optimum fitness, of course.
However we argue in section \ref{sec:dec_fitness} that in most combinatorial problems it also holds for a wide range of fitness levels.
The proofs that neighbourhood search is beneficial rely on two lemmas which allow some weakening of this conditions.
In sections \ref{sec:2sat} and \ref{sec:tspexperiments} we show the lemmas hold for fitness levels quite far from the optimum.

The results of our experiments also show that when the decrease in fitness probabilities is steep - which is the case for many combinatorial problems - neighbourhood search dramatically outperforms blind search.

\subsection{Goal of this paper}
Much of the previous work in local search investigates how the performance of different algorithms is influenced by the landscape structure.
However, the properties required for the results of this paper say nothing about local optima, or basins of attraction.
Nor does this paper address the challenge of selecting a local search algorithm with which to tackle a given problem.
In fact a key contribution of this paper is an abstraction of local descent, introduced in section \ref{sec:abstract-local-descent}, in which there are no local optima.

In short, the goal of this paper is to validate two key assumptions of local search.
The motivation is theoretical, establishing some fundamental reasons why local search works.

\section{Related Work}
Over 30 years ago, Johnson et.al. asked ``How easy is neighbourhood search?'' \cite{easy}.  They investigated the complexity of finding locally optimal solutions to NP-hard combinatorial optimisation problems.  They show that, even if finding an improving neighbour (or proving there isn't one) takes polynomial time, finding a  local optimum can take an exponential number of steps.

\cite{tovey1985} considers hill-climbing using flips of $n$ zero-one variables.  If the objective values are randomly generated the number of local optima tends to grow exponentially with $n$.  Thus the expected number of successful flips to reach a local optimum from an arbitrary point grows only linearly with $n$.  A more general analysis of neighbourhood search in \cite{tovey-3} explores a variety of algorithms to reach a local optimum, but does not compare neighbourhood search against random generate-and-test. 
In Tovey's scenario, with only zero-one variables where the neighbourhood is defined by single flips, there are problems that require an exponential number of flips to reach a local optimum - even choosing the best neighbour each time~\cite{cohen2020steepest}.

\cite{grover1992} analysed the average difference between a candidate solution and its neighbours, for five well-known combinatorial optimisation problems.  Since this difference is positive for candidates with less than average fitness it implies that any local optima must have a better than average fitness.  The result also shows how many moves it requires from an arbitrary initial value to reach a better than average fitness.    However this result is for a small set of problems and tells us little about local descent from a candidate solution of better than average fitness.

The idea of counting the number of solutions at each fitness level is directly related to the concept of the \emph{density of states} which applies to continuous fitness measures encountered in solid state physics \cite{fitness_distribution}.
This work additionally shows how to estimate the density of states for a problem using Boltzmann strategies.

Neighbours' similar fitness (NSF) is related to the idea of fitness correlation, spelled out for example in the paper ``Correlated and Uncorrelated Fitness Landscapes and How to Tell the Difference'' \cite{ar1}.
In this paper random tours (each pair of points on the tour being neighbours) are used to predict the fitness of a point as a linear combination of the fitness of preceding points on the tour. 
Another related concept is \emph{fitness distance correlation} \cite{fdc}, which is a measure in a landscape of the correlation between the fitness of a point and its distance from the nearest global optimum.
FDC is a measure designed to aid a particular class of algorithms, dependent on a global property of the landscape. 

These notions of correlation are tied to the landscape structure, in contrast to NSF which only applies to the immediate neighbours of points with a given fitness.
NSF holds of many well-known neighbourhood operators, and we suggest that it is even a criterion used in designing such operators.

\section{Properties needed for local descent to be beneficial}
In this paper, for uniformity, we assume that optimisation is cost minimisation. 
For historical reasons we have until now used ``fitness'' as the property of a solution to be optimised, however in case of minimisation, where ``fitness'' is misleading, we shall use ``cost'' as a synonym for ``fitness''.
We assume a finite range of fitness values, and without loss of generality, we set the optimum fitness to be $0$.

This section addresses black box optimisation, where the problem is known to belong to a given class.  This subsumes white box optimisation, which corresponds to black box optimisation in a problem class of size one.

\subsection{Definitions}
The probability a point, in a given problem class, has a given fitness is its fitness probability.
The expression {\em neighbour of fitness $k$} means neighbour of a point with fitness $k$.
The neighbour's fitness probability is the probability that a neighbour of fitness $k_1$ has a given fitness $k_2$.
Like fitness probabilities, neighbours' fitness probabilities are a property of problem classes rather than problems.

\subsubsection{Neighbourhood search symbols}
\label{local_search_intro}
We introduce the following definitions:
\begin{itemize}
\item A problem class $PB$ comprises problems $pb \in PB$, where each problem $pb$ has the same search space $S$ but a different fitness function $f_{pb}(s):s \in S$ which takes values in the same finite fitness range.
\item  The fitness range of a problem class is $0 \ldots k_{max}$, where $0$ is the optimal fitness (or cost), and $k_{max}$ the worst.
\item $p(k)$ is the probability of a fitness $k \in 0 \ldots k_{max}$, in a problem class.
Specifically, the fitness probability $p(k)$ is the probability that an arbitrary point $s \in S$, in an arbitrary problem $pb \in PB$, has $f_{pb}(s)=k$ 
\item The probability that a neighbour of a fitness $k_1$ has fitness $k_2$ in a problem class is $pn(k_1,k_2)$.\\
Specifically, the neighbourhood probability $pn(k_1,k_2)$ is the probability that, given an arbitrary problem $pb \in PB$, and given an arbitrary point $s$ with $f_{pb}(s) =k_1$, 
an arbitrary neighbour $s'$ of $s$ has $f_{pb}(s')=k_2$.
\item $p(k \pm \delta) = p(k+\delta)+p(k-\delta)$.  
Similarly for $pn(k, k \pm \delta)$.
\end{itemize}

\subsection{Problem classes and the no free lunch theorem}
\label{sec:sharpnofree}
In \cite{joyce-herrmann} a fitness function is specified by ordering the search space, so $S = \{1, \ldots, n\}$, and writing out the list of fitness values $< f(1), \ldots, f(n)>$.
If $\phi$ is a permutation of $\{1, \ldots, n\}$, then $f_{\phi}(s) = f(\phi(s))$ is a permutation of the function $f$.

For any problem class, with a given solution space $S$ and finite fitness range $K$, which is closed under permutation, the sharpened no free lunch theorem of  \cite{joyce-herrmann}, Theorem 11, applies. 
The key insight is that because all permutations of each possible function are in the set, even if it is known how many of the untested points in the search space take each fitness value, every permutation of the future points and their values remains equally likely.
This makes it impossible to predict which of these future points take which values.
Thus including the resulting fitness probabilities $p(k): k \in K$ does not invalidate, or in any way change, the proof of the sharpened no free lunch theorem.

Indeed given an instance of a combinatorial problem, its search space and its fitness function $f$, 
We could also create a problem class $Perm_f$ comprising problems with the fitness function $f$ and all its permutations.
$Perm_f$ has the same fitness probabilities $p(k): k \in K$ as our original combinatorial problem.
$Perm_f$ satisfies the conditions of the sharpened no free lunch theorem and consequently no local search algorithm can outperform blind search on this problem class.

\cite{joyce-herrmann} show that the behaviour of any optimisation algorithm applied to the problem can be represented as a \emph{trace}: $<s_1, f(s_1)>, <s_2,f(s_2)>, \ldots, <s_t,f(s_t)>$, where $t$ is the number of steps taken by the algorithm.
The performance of the algorithm is some function of the sequence of fitness values, 
$f(s_1),f(s_2), \ldots, f(s_t)$, for example their minimum $min(\{f(s_1),f(s_2), \ldots, f(s_t)\})$.

In this paper the behaviour of search algorithms will also be representable as a trace, but of fitness values $k_1, k_2 \ldots k_t$ only.
The performance of such an algorithm can be defined in the same way as by \cite{joyce-herrmann}, for example the minimum $min(\{k_1, k_2, \ldots, k_t\})$.
The behaviour of blind search on a problem class can also be represented as a sequence of fitness values.

This paper will show that the NSF property introduced in this paper supports a proof that under certain circumstances - which can hold in $Perm_f$ for example - local search outperforms blind search.
Thus NSF is sufficient to escape the conditions of the no free lunch theorems.

\subsection{Starting fitness}
\label{sec:dec_fitness}
We are interested in problem classes whose fitness probabilities decrease towards the optimum. 
Specifically such problem classes have a moderate fitness level, $k_{mod}$, better than which this thinning out occurs.
\begin{definition}
The fitness level $k_{mod}$ is the worst fitness better than which $p(k)$ is monotonically decreasing\footnote{In this paper we use \emph{monotonically decreasing} to be synonymous with \emph{monotonically nonincreasing}, and similarly \emph{monotonically increasing} means \emph{monotonically nondecreasing}} with increasing fitness. 

In this paper the optimum fitness is the minimum cost, so 
$\forall k_1 \leq k_2 \leq k_{mod} : p(k_1) \leq p(k_2) $
\end{definition}
For many problems 
$k_{mod}$ lies about halfway between $0$ and $k_{max}$.
However, for a problem class whose fitness probabilities are uniform ($\forall i,j :  p(i)=p(j)$), the modal cost is the maximum fitness $k_{mod}=k_{max}$.

Neighbourhood search is unlikely to improve any faster than blind search starting from a very poor fitness.  
We therefore consider neighbourhood search from a current solution which is already of reasonably good fitness. 
Nevertheless it should be a level of fitness that could be reached quickly with a simple initialisation procedure, such as blind search (see section \ref{sec:blindthenlocal} below).
Our proof of the benefit of neighbourhood search, requires that in the current neighbourhood, the fitness probability should be decreasing with improving fitness.

Specifically if $k_c$ is the current cost, we require that $p(k)$ should be monotonically increasing with $k$ in the range $0 \ldots 2*k_{c}$.
For such a cost $k_c$, $ p(k_c+\delta): \delta \leq k_c$ monotonically increases with increasing $\delta$, while $p(k-\delta)$ decreases.

\begin{definition}[GE]
We define \emph{good enough cost} $k_{ge}$ as:
\begin{equation}
k_{ge} = k_{mod} / 2  \label{eq:ge} 
\end{equation}
\end{definition}

Some problem classes with a reasonably high value of $k_{mod}$ include PLS-complete problems, and problems in the classes listed in \cite{michiels2007}, Appendix C, where the local optima thin out sharply towards the global optimum.
Their cost functions are  expressed as the sum of a set of terms, where
each term has a small set of variables, whose number is independent of the size of the problem instance.  
Only the number of terms in the cost function increases with the number of variables in the problem instance.
For such a function the optimum is only reached when all the terms take their minimum value: there is just one such point.
There are ${n \choose x}$ ways that $x$ out of $n$ terms take their minimum value, and this number increases by a factor of $\frac{n-x}{x}$ when $x$ decreases by one.
Thus, as the cost increases away from the optimum, the number of combinations of values that reach that sum increases dramatically, thus increasing its probability.
Notice that in this case $k_{mod}=n/2$: when $x$ is less than this $\frac{n-x}{x}$ is greater than one, but beyond this $\frac{n-x}{x}$ is less than one.

For VLSI problems, for example, \cite{white1984} showed that the solution fitnesses have a normal distribution over the interval between their minimal and maximal fitness, having few solutions with fitness near the extremes.

Combinatorial problems with \emph{hard} constraints, that cannot be violated in any solution, also typically have a high value of $k_{mod}$.
In section \ref{sec:tspexperiments} we confirm this experimentally on random sets of travelling salesman problems.

\subsection{Neighbours Similar Fitness at level k (NSF(k))}
\label{sec:nsf}
We have already given examples of neighbourhoods designed for local descent, such a 2-opt, 
where neighbours have similar fitness.
Indeed if there was no correlation between the fitness of neighbours,
the fitness of a randomly selected neighbour would have the same probability as the fitness selected by blind search.
In this case neighbourhood search would be no better than blind search.

If the current fitness $k$ is near the optimum, similar fitnesses may have a very low fitness probability.  At slightly worse fitness levels, on the other hand, the fitness probability may be much greater.  These  probabilities are also likely to be reflected in the neighbourhood of fitness $k$.

Consequently, rather than the probability a neighbour has similar fitness, $NSF(k)$ describes the \emph{increased} probability a neighbour of a point with fitness $k$ has similar fitness, over the probability an arbitrary point in the search space has similar fitness.

We introduce the function $r(k,\delta)$
which is a weighting associated with fitness distance $\delta$ for points in the neighbourhood of any point with fitness $k$.
Accordingly
$pn(k,k \pm \delta) = p(k \pm \delta) \times r(k,\delta)$.
We call \emph{r} the {\em NSF weight}.

\begin{definition}[NSF weight]
The NSF weight $r(k,\delta)$ for fitness $k$ and fitness difference $\delta > 0$ gives probability a neighbour of a point with fitness $k$ has fitness $k \pm \delta$:
$$ pn(k,k \pm \delta) = p(k \pm \delta) \times r(k,\delta) $$
\end{definition}

A second key property of NSF is that the neighbours should not all have worse fitness than the current solution. Indeed the proportion of worse solutions should be no greater than the proportion in the search space as a whole.
We call such a neighbourhood \emph{normal}.
For a given cost (fitness) $k$, let us call $r(k,\delta) \times p(k - \delta)$ the NSF weighted probability of an improving cost $k - \delta$.
\begin{definition}[normal]
The neighbourhood of fitness $k$ is normal if neighbourhood probability of any improving cost $pn(k,k-\delta)$ is no less than its NSF weighted probability:\\
\begin{equation}
\forall k, \delta: pn(k,k-\delta) \geq r(k,\delta) \times p(k-\delta) 
\label{eq:nonskew}
\end{equation}
\end{definition}

The Neighbourhood Similar Fitness $NSF(k)$ property is defined as follows.
\begin{definition}
$NSF(k)$ holds if the the neighbourhood of $k$ is normal, and the NSF weight $r(k,\delta)$ increases as $\delta$ decreases: 
\begin{align*}
\forall \delta : pn(k,k-\delta) \ge r(k,\delta) \times p(k-\delta)\\
\forall \delta_1 \leq \delta_2 : r(k,\delta_1) \geq r(k,\delta_2) 
\end{align*}
 \label{eq:nsf}
\end{definition}

\subsection{NSF(k) in practice}
For solving many combinatorial optimisation problems, 
a neighbourhood operator is often chosen that changes the value of only a few terms from the objective function.
Assuming the objective is the sum of many such terms, neighbours are likely to have similar fitness.

For example, the maximum satisfiability problem is to find an assignment to truth variables
that minimises the number of unsatisfied clauses. 
A clause is just a disjunction of some truth variables (or none) and some negated truth variables (or none).
One way of finding a neighbour for this problem is by changing the truth value
of one variable (“flipping” a variable). 
If the clause length is restricted to just three (variables and negated variables),
and if there are, say, 100 truth variables in the problem, then
only $1-\frac{99^3}{100^3} = 0.03$ of the clauses are likely to contain any given variable. 
Thus after flipping a variable $97$\% of the clauses will remain unchanged.
Consequently the fitness of a neighbour found by flipping a variable is likely to be similar
to the original fitness.
The same argument applies to most problems whose fitness function is a sum of terms in which the number of
terms increases with the number of variables in the problem.
Any neighbourhood operator that changes the value of a single variable, or a small set of variables, will only change the value of a small fraction of the terms in the
sum.
Consequently neighbours are likely to have similar fitness. 

Interestingly, this characteristic is one that applies to many \textit{PLS-complete} problems~\cite{easy}. 
For further information on this class, we refer the interested reader to Michiels et.al. \cite{michiels2007} and Sch\"{a}ffer et.al. \cite{Schaffer1991}, who list many PLS-complete optimisation problems that all have objectives that are a sum of terms which grow with problem size; among these are \textsc{MaxCut} with flip neighbourhood, \textsc{GraphPartitioning} with swap neighbourhood, \textsc{MetricTSP} with Lin-Kernighan neighbourhood, and \textsc{CongestionGame} with switch neighbourhood.  
Other investigations have covered the maximisation of submodular functions~\cite{Feige2007submod}, and the difference between submodular functions~\cite{Iyer2012differenceSubmodular}.

Although, in a continuous search space, the objective function typically has a gradient at most points, in a discrete search space there is no such thing as a gradient.  
There are ways of extending the notion of continuity to fitness functions that are discrete.
An example is the Lipschitz condition 
$|f(u) - f(w)| \leq L \times ||u - w||$ for some constant $L$.
Interpreting $||u-w|| = 1$ for any pair of neighbours, $u,w$, the Lipschitz condition simply imposes a bound on the difference $|f(u)-f(w)|$.
However neighbourhoods used in combinatorial optimisation, such as the 2-opt operator, do not have a fixed upper bound on the fitness difference between neighbours.
Moreover even in case there is a bound on the fitness difference between neighbours, NSF relates fitness differences $v=f(u)-f(w)$ below this bound to their probability for different values of $v$.
Thus the Lipschitz condition does not capture the property of neighbourhood operators that we require for neighbours similar fitness.

Secondly, in combinatorial optimisation problems there is no way to elicit an improving direction, or intelligently to guess whether a neighbour of a given point is improving - except by evaluating it.
This is why in our model of local improvement and local descent, a neighbour is chosen from the neighbourhood at random.

A neighbourhood is \emph{boosting} if $\forall \delta : pn(k,k-\delta) > r(k,\delta) \times p(k-\delta)$.  
In this case the neighbours of a point are skewed towards improving fitness, so local search is likely to work well.
In practice, for good levels of cost neighbourhoods are often boosting, while for high levels of cost the reverse is true.
If $k_{bad}$ is a high cost, then $\forall \delta : pn(k_{bad},k_{bad}+\delta) > r(k_{bad},\delta) \times p(k_{bad}+\delta)$.
The key theorems in the paper are proven for normal neighbourhoods, but the benefits of local search,
are even greater for boosting neighbourhoods.

In section \ref{sec:tspexperiments} we confirm that $NSF(k)$ holds for some travelling salesman problems (TSP).  
The TSP is a typical combinatorial optimisation problem.

\subsection{An example problem class: 2-SAT}
\label{sec:2sat}
In this section we take a very simple example of a problem class, and show how we can infer its specification and properties.

We calculate $p$, $pn$ and $r$ for the set of 2-SAT problems with $50$ variables and $100$ 2-variable clauses, in which each variable appears in exactly $4$ distinct clauses.
Each variable can take the value $true$ or $false$, so there are $2^{50}$ candidate solutions.
The fitness of a solution is the number of violated constraints, so the range of fitness values is $0 \ldots 100$.
This completes the specification of our example problem class.

Based on the above specification we calculate $p$, $k_{mod}$, $pn$ and $r$.
The specification of these calculations is given in \ref{app:2SAT}.

Each clause is true with probability $3/4$ and false with probability $1/4$.
The probability that $C$ clauses are false is 
$$p(C) = (1/4)^C \times (3/4)^{(100-C)} \times {C \choose 100}$$
The most likely cost is $p(25) = 0.092$, and this is the modal cost $k_{mod}$.
The probabilities are shown in figure \ref{fig:sat2p}

\begin{figure}[htb]
\centering
\begin{tikzpicture}
\begin{axis}[
width=5.5cm,
title= Probability at each cost level,
scaled ticks = false,
xlabel={Cost: $i$},
ylabel={Probability: $p(i)$},
y tick label style={
                /pgf/number format/fixed,
                /pgf/number format/fixed zerofill,
                /pgf/number format/precision=2},
]
\addplot [blue] [solid]  table {sat2p.dat};
\end{axis}
\end{tikzpicture}
\caption{SAT-2 probability at each fitness level}
\label{fig:sat2p}
\end{figure}
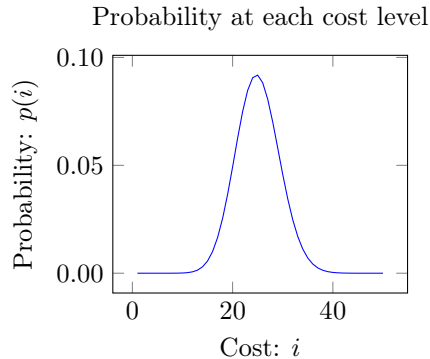

The neighbours of a solution result from flipping the value of a single variable.  Since a variable only appears in $4$ clauses,
$pn(C,\delta)=0$ for all $C$ and any $\delta>4$.
Flipping a variable in a clause that is false always makes it true - thus increasing the cost by $1$.
Flipping a variable in a true clause makes it false with a probability of $1/2$.
If the current fitness is $C$, the probability that the variable to be flipped is in a false clause is $C/100$, and a for true clause it is $(100-C)/100$.  Thus:
$$pn(C,C-4) = (C/100)^4$$
and
$$pn(C,C+4) = (((100-C)/100) \times 1/2)^ 4$$
The expressions for $pn(C1,C2)$ 
are all given in \ref{app:2SAT}.

For this problem class and neighbourhood operator, neighbourhoods are boosting for cost values lower than the modal cost $25$.
The values for $r$, $p$, $r \times p$ and $pn$ are given for cost $20$ in table \ref{tab:2satunskew}.
\begin{table}[!ht]
    \centering
    \begin{tabular}{|c||c|c||c|c|}
    \hline
          &  &  & $r(20,\delta) \times$ &  \\
        $\delta$ & $r(20,\delta)$ & $p(20- \delta)$ &  $p(20- \delta)$ & $pn(20,20- \delta)$ \\
    \hline     
         1 & 3.87 & 0.037 & 0.141 & 0.256 \\
         2 & 2.55 & 0.025 & 0.065 & 0.205 \\
         3 & 1.14 & 0.017 & 0.019 & 0.102 \\
         4 & 0.27 & 0.010 & 0.003 & 0.026 \\
    \hline
    \end{tabular}
    \caption{Cost $20$ probabilities $r, p, r \times p, pn$}
    \label{tab:2satunskew}
\end{table}

The table also shows that $r$ decreases with increasing $\delta$.
Indeed this NSF property holds for all fitness levels.

 \section{Probability of Improvement}
 
\subsection{Beneficial neighbourhood search}
The properties introduced in the previous section together entail that the ratio of better to worse NSF weighted probabilities is higher than the ratio of better to worse fitness probabilities.
Let us write $\bar{r}(k)$ for the average value of $r(k,\delta): \delta \in 1..k$.
We call it the average NSF weight:
$$\bar{r}(k) = \sum_{\delta=1}^k r(k,\delta) / k$$
Let us write $p^<(k)$ for the probability blind search selects a point better than $k$.  
We call it the blind probability of improving:
$$p^<(k) = \sum_{\delta=1}^k p(k-\delta)$$
Let us write $pn^<(k)$ for the probability neighbourhood search, starting from fitness $k$, selects a point better than $k$. 
We call it the neighbourhood probability of improving:
$$pn^<(k) = \sum_{\delta=1}^k pn(k,k-\delta)$$
Thirdly let us write $pbr^<(k)$ for the NSF weighted probability of selecting a neighbour with cost lower than $k$.
We call it the NSF weighted probability of improving:
$$pbr^<(k) = \sum_{\delta=1}^k p(k-\delta) \times r(k,\delta)$$
For the probability of blindly selecting a worse point we write:\\
$$p^>(k) = \sum_{\delta=1}^{k_{max}-k} p(k+\delta)$$
and similarly for $pbr^>(k)$.

\begin{lemma}
\label{lemma:rbarimp}
Conditions sufficient that NSF weighting improves the probability of improving over the probability of worsening.\\
Assuming
\begin{flalign*}
\ \ \ \ pbr^<(k) \geq & \ \bar{r}(k) \times p^<(k)& \\
\ \ \ \ pbr^>(k) \leq & \ \bar{r}(k) \times p^>(k)& 
\end{flalign*}
it follows that
\begin{flalign*}
\ \ \ \ \frac{pbr^<(k)}{pbr^>(k)} \geq & \ \frac{p^<(k)}{p^>(k)}&
\end{flalign*}
\end{lemma}

\begin{lemma}
\label{lemma:improvement}
NSF neighbourhoods and a good enough starting fitness are sufficient to entail the conditions of lemma \ref{lemma:rbarimp}.\\
Assuming
\begin{flalign*}
\ \ \ \ k \leq & \ k_{ge} &(GE)\\
\ \ \ \ \forall \delta_1 < & \ \delta_2 : r(k,\delta1) \geq r(k,\delta2) &\ \ (r \  {\rm decreasing})) 
\end{flalign*}
it follows that
\begin{flalign*}
\ \ \ \ pbr^<(k) \geq & \ \bar{r}(k) \times p^<(k) &\\
\ \ \ \ pbr^>(k) \leq & \ \bar{r}(k) \times p^>(k) &
\end{flalign*}
\end{lemma}
The proof
is in \ref{sec:improve-proof}

This result only refers to fitnesses strictly better and worse than the current fitness $k$, but nothing about the fitness $k$ itself.
Indeed the conditions admit neighbourhoods for which all neighbours of $k$ have the same fitness, $k$.
Neighbourhood search on such a neighbourhood will never yield any improvement in fitness.

We can give a simple condition that ensures the NSF weighted probability of improvement is greater than blind search.  
The condition limits the probability that neighbours have the same fitness.
\begin{lemma}
\label{lemma:pkk}

Assuming
\begin{flalign*}
\ \ \ \ pbr^<(k) \geq & \ \bar{r}(k) \times p^<(k) & \\
\ \ \ \ pbr^>(k) \leq & \ \bar{r}(k) \times p^>(k) &\\
\ \ \ \ pn(k,k) \leq & \ p(k) &
\end{flalign*}
it follows that
\begin{flalign}
\label{eq:improvement}
\ \ \ \ pbr^<(k) \geq & \ p^<(k) &
\end{flalign}
\end{lemma}
\noindent The proof is in \ref{sec:improve-proof}.

Even if there are more neighbours with the same fitness, as long as $\bar{r}(k) \geq 1$, the NSF weighted probability of improving is higher than the blind probability of improving:
\begin{lemma}
\label{lemma:improver}
Assuming
\begin{flalign*}
\ \ \ \ pbr^<(k) \geq & \ \bar{r}(k) \times p^<(k) &\\
\ \ \ \ \bar{r}(k)\geq & \ 1 & 
\end{flalign*}
it follows that
\begin{flalign}
    \label{eq:imrovementr}
\ \ \ \ pbr^<(k) \geq & \ p^<(k) &
\end{flalign}
\end{lemma}
\noindent (by lemma \ref{lemma:decdec}, in \ref{sec:improve-proof}).\\
If $pn$ is normal or boosting, $pn^<(k) \geq pbr^<(k)$.

The key theorem of this section follows immediately:
\begin{theorem}
\label{thm:improvement}
Neighbourhood search has a greater probability of improving than blind search.\\
Assuming
\begin{flalign*}
&k \leq  k_{ge}& &(GE)\\
&\forall \delta_1 <\delta_2 : r(k,\delta1) \geq r(k,\delta2)& &(NSF(k)) \\
&\forall \delta : pn(k,k-\delta) \geq r(k,\delta) \times p(k-\delta)& &\\
\rm{either \ } &p(k) \geq pn(k,k) & &\rm{(lemma \ } \ref{lemma:pkk})\\
\rm{or \ }  &\bar{r}(k) \geq 1& &\rm{(lemma \ } \ref{lemma:improver})
\end{flalign*}
it follows that
\begin{flalign*}
\ \ \ \ pn^<(k) \geq \ & p^<(k)&
\end{flalign*}
\end{theorem}
Finally we note that whenever $r(k,1)>1$, and there are any points with worse than modal fitness, then neighbourhood search is strictly superior to blind search:
\begin{flalign*}
\ \ \ \ pn^<(k) > \ & p^<(k)&
\end{flalign*}
This is proven as lemma \ref{lemma:strict} in \ref{sec:improve-proof}

In \ref{sec:counterex} it is shown that the absence of each condition, enables a counter-example to theorem \ref{thm:improvement} to be constructed.

\paragraph{2-SAT example again}
The focus of this work is on combinatorial problems where a neighbourhood can only be searched blindly, selecting and evaluating points in an arbitrary order.
The simple 2-SAT example is not a combinatorial problem in this sense (indeed it can be optimised by a polynomial algorithm).
Nevertheless we will complete the calculation of features of interest for this problem class.

In the 2-SAT example above
$$r(10,1),r(10,2),r(10,3),r(10,4) = 1162,461,116,14.$$
The average value of $r$ is
$(1162+461+116+14) / 10 = 175$\\
This satisfies the improvement condition very strongly!

Indeed the probability of cost lower than $10$ is
$$\sum_{\delta=1}^{10} p(10-\delta) = 0.000043$$ 
while the NSF weighted probability of improvement is
$$\sum_{\delta=1}^{10} r(10,\delta) \times p(10-\delta) = 0.04$$
which is several orders of magnitude greater.
In fact at this cost the neighbourhood is strongly boosting, and replacing $r(10,\delta) \times p(10-\delta)$ with $pn(10,10-\delta)$ we get an even higher probability of improving:
$$\sum_{\delta=1}^{10} pn(10,10-\delta) = 0.76$$

We also note here that $pn(10,10) = 0.16$ while $p(10) = 0.000094$.  
Clearly $pn(k,k)<p(k)$ alone is not a necessary improvement condition.

In our 2-SAT example, the largest cost $k$ for which
$\bar{r}(k) >  1$ 
is at $k=17$.
However, as the neighbourhood operator is strongly boosting and the cost probability sharply declining, the largest cost for which
$$\sum_{delta=1}^k pn(k,k-\delta) >  \sum_{\delta=1}^k p(k-\delta)$$
is $k=25$. 
In this example (blind) neighbourhood search outperforms blind (unrestricted) search only for costs lower than $25$.

\subsection{Benchmarks for the probability of improvement}
In this section we explore the probability of improvement over a set of benchmark problem classes.  The probability of improvement for neighbourhood search is compared with random search, revealing how the improvement grows with the rate at which the cost probability $p(k)$ falls as $k$ decreases and the rate at which the NSF weight $r(k,\delta)$ falls as $\delta$ increases

\label{sec:bench}
We created three problem classes, where the cost probability falls increasingly fast:
\begin{itemize}
    \item uniform - in which the cost probability at each cost is the same
    \item linear - in which the cost probability decreases linearly towards the optimum
    \item exponential - in which the cost probability decreases exponentially near the optimum
\end{itemize}
In these problem classes $k_{opt}=0$, $k_{max}=200$.  The moderate cost $k_{mod}=100$.
Their cost probabilities are shown in figure \ref{fig:probcts}.
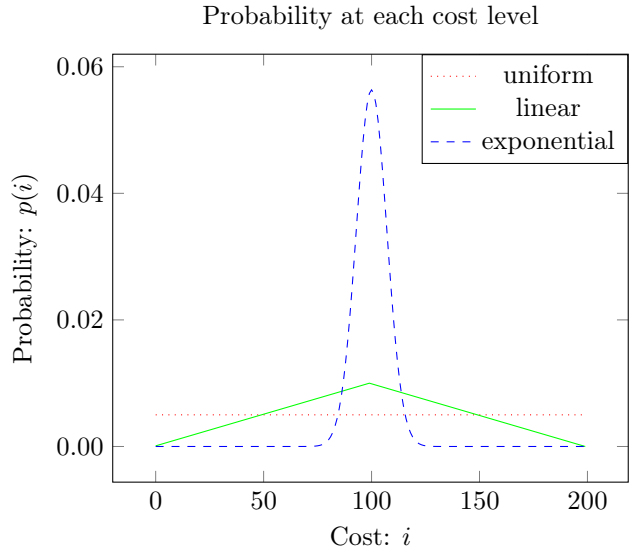
\begin{figure}[htb]
\centering
\pgfplotsset{
    every axis legend/.append style={
    at={(1,1)},
    anchor=north east,
                                    },
            }
\begin{tikzpicture}
\begin{axis}[
title= Probability at each cost level,
scaled ticks = false,
xlabel={Cost: $i$},
ylabel={Probability: $p(i)$},
y tick label style={
                /pgf/number format/fixed,
                /pgf/number format/fixed zerofill,
                /pgf/number format/precision=2},
legend entries={uniform,linear,exponential},
]
\addplot [red] [dotted] table {const.dat};
\addplot [green] [solid]  table {linear.dat};
\addplot [blue] [dashed]  table {expo.dat};
\end{axis}
\end{tikzpicture}
\caption{Benchmark probability at each cost level}
\label{fig:probcts}
\end{figure}

For the linear problem class $p(k-1) = p(k)-C$, for some positive constant $C$, so the cost probability falls faster.
In the exponential problem class the difference between $p(k)$ and $p(k-1)$ grows rapidly as $k$ nears the modal cost $100$.

In the following benchmarks, the probabilities at each cost level are given in table \ref{tab:bench_cts} (choosing $C=1$):
\begin{table}[!ht]
    \centering
    \begin{tabular}{|c|c|c|}
    \hline
         uniform & linear & exponential \\
         $1/200$ & $(100 - abs(100-i))/10100$ & ${200 \choose i+1}/2^{200}$\\
         \hline
    \end{tabular}
    \caption{Calculation of $p(i)$}
    \label{tab:bench_cts}
\end{table}

We also explore different types of NSF weighting.  Again we choose one example where $NSF(k)$ is satisfied minimally ($r(k,\delta)=1$ everywhere), and others where it is satisfied increasingly strongly, shown in figure \ref{fig:nsfweights}.
In these NSF weightings, the fixed bound $b$ is the maximum fitness difference between any neighbours. 
For any fitness difference $\delta > b$, the NSF weight $r(k,\delta) = 0$.
For smaller fitness difference the NSF weight is a fixed value (such that the sum of probabilities of neighbours of all costs is 1).

Figure \ref{fig:nsfweights} shows the NSF weights around a cost of 50, which for cost-difference $i$ is $r(50, \rm{abs}(i-50))$.
The figure shows the values for $b=10, 50,$ and $200$.  
Notice that if $b \geq 200$, then $pn(k,k2)=p(k2)$ everywhere.

\begin{figure}[!ht]
\pgfplotsset{
    every axis legend/.append style={
    at={(1,1)},
    anchor=north east,
                                    },
            }
\begin{tikzpicture}
\begin{axis}[
title= NSF weight at each cost level,
xlabel={Cost: {i}},
ylabel={NSF weight: $r(50,\rm{abs}(i-50))$},
legend entries={b=10,b=50,b=200},
]
\addplot [black] [dotted]  table {rkiconst10.dat};
\addplot [green] [solid]  table {rkiconst50.dat};
\addplot [blue] [dashed]  table {rkiconst200.dat};
\end{axis}
\end{tikzpicture}
\caption{Benchmark NSF weight at each distance}
\label{fig:nsfweights}
\end{figure}
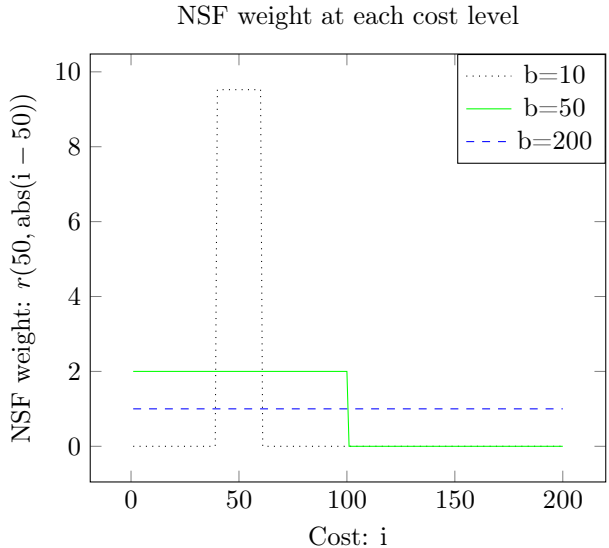

With these different problem classes and different NSF weightings, we then calculated the probability of improvement under neighbourhood search, compared with blind search.
We calculated $\sum_{k2<50} p(k2)$  and $\sum_{k2<50} pn(50,k2)$ for each benchmark,
shown in table \ref{tab:impbench}.
The three rows are for the uniform, linear and exponential benchmarks.
The five columns are for different bounds $b$ such that 
$j>b \rightarrow r(50,j)=0$.
The table shows firstly that for $b=200$ neighbourhood search is no better (or worse) than random search.
Secondly, comparing the three rows, the table shows how a steeper fall in cost probability yields an increased probability of improvement with neighbourhood search.
\begin{table}[!ht]
\centering
    \begin{tabular}{| c | c | c c c c c |}
    \toprule
 & $\sum_{k2<50}p(k2)$ & \multicolumn{5}{c |}{$\sum_{k2<50}pn(50,k2)$}\\     
& & $b=1$ & $b=5$ &  $b=10$ & $b=50$ & $b=200$ \\ 
    \midrule
uniform   & 0.25 &  0.33 & 0.45 & 0.48 & 0.49 & 0.25\\
linear & 0.16 &  0.33 & 0.44 & 0.45 & 0.33 & 0.16 \\
exponential  & 0.06 &  0.34 & 0.43 & 0.37 & 0.12 & 0.06 \\
    \bottomrule
  \end{tabular}
  \vspace{1ex}
  \caption{
  }
  \label{tab:impbench}
\end{table}

The impact of the NSF strength is illustrated in figure \ref{fig:impk}.
The figure maps the probability of improvement against the starting cost $k$.  
It shows how $pn^<(k)$ changes as $k$ is reduced from $50$ to $1$ on the uniform, linear and exponential problem classes.
For the exponential problem class $pn^<(k)$ is so much greater than $p^<(k)$ that the probability axis in the graph is shown on a logarithmic scale.

Clearly if $k$ is optimal (in this case 0), the probability of improvement is also $0$. 
As the starting cost grows worse, the $b=200$ line shows how for blind search the probability of improvement grows.  For the uniform problem class, this growth is linear, and for the linear problem class it follows a quadratic curve.

For a tighter NSF, with $b=50$ the probability of improvement grows more quickly, and for $b=10$ it grows very quickly.  Indeed, for the uniform problem class, the probability of improvement with $b=10$ reaches 50\% at a starting cost of $10$, and stays there.

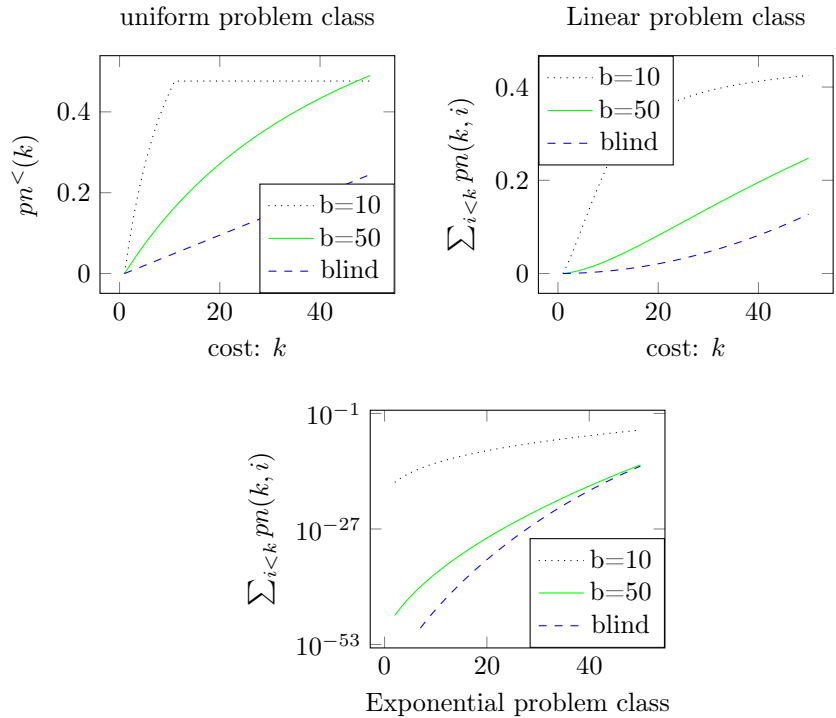
\begin{figure}[ht]
\pgfplotsset{
    every axis legend/.append style={
    at={(1,0)},
    anchor=south east,
                                    },
            }
\begin{tikzpicture}
\begin{axis}[
width=5.5cm,
title= uniform problem class,
xlabel={cost: $k$},
ylabel={$pn^<(k)$},
legend entries={b=10,b=50,blind},
]
\addplot [black] [dotted] table {impconst10.dat};
\addplot [green] [solid]  table {impconst50.dat};
\addplot [blue] [dashed]  table {impconst200.dat};
\end{axis}
\end{tikzpicture}
\hskip 10pt
\pgfplotsset{
    every axis legend/.append style={
    at={(0,1)},
    anchor=north west,
                                    },
            }
\begin{tikzpicture}
\begin{axis}[
width=5.5cm,
title= Linear problem class,
xlabel={cost: $k$},
ylabel={$\sum_{i<k} pn(k,i)$},
legend entries={b=10,b=50,blind},
]
\addplot [black] [dotted] table {implinear10.dat};
\addplot [green] [solid]  table {implinear50.dat};
\addplot [blue] [dashed]  table {implinear200.dat};
\end{axis}
\end{tikzpicture}

\begin{center}
\pgfplotsset{
    every axis legend/.append style={
    at={(1,0)},
    anchor=south east,
                                    },
            }
\begin{tikzpicture}
\begin{axis}[
width=5.5cm,
ymode=log,log basis y=10,
xlabel={Exponential problem class},
ylabel={$\sum_{i<k} pn(k,i)$},
legend entries={b=10,b=50,blind},
]
\addplot [black] [dotted] table {imp100expo10.dat};
\addplot [green] [solid]  table {impexpo50.dat};
\addplot [blue] [dashed]  table {impexpo200.dat};
\end{axis}
\end{tikzpicture}
\end{center}

\caption{Probability of Improvement}
\label{fig:impk}
\end{figure}

\subsection{Testing NSF(k) on combinatorial problems }
\label{sec:tspexperiments}

This section generates random TSP problem classes and measures the solution count at different cost levels, and the cost of their neighbours.  Based on these numbers we predict their NSF weights and thus check the NSF property at these levels.  

The neighbourhood operator used in our experiments is the 2-opt operator \cite{twoswap} which changes only two
distances in the sum for symmetric TSPs – no matter how
many cities there are in the TSP. This is, in fact, the smallest
change possible (in terms of the number of distances
changed) while maintaining the constraint that the tour must
be a Hamiltonian cycle. By ensuring that $n-2$ distances remain
the same (where $n$ is the number of cities), the 2-opt
generates neighbours which are in general likely to have similar cost
to the current tour - at least more similar than a randomly generated tour.

For a small TSP with 10 cities, the inter-city edge lengths are randomly generated in the range 1 to $20$.
All the solutions are found, and all their neighbours, at every fitness level.  

For a larger TSP, with 100 cities, the inter-city edge lengths are randomly generated in the range 1 to $100$. 
Two million solutions are sampled, and based on these solutions the good enough fitness level, $k_{ge}$ is estimated.  The fitness of two million points are then checked, and for those with fitness $k_{ge}$ all their neighbours are found.  Thus $pn(k_{ge},\delta)$ is estimated for all values of $\delta$ in the range $1 \ldots k_{ge}$.

These values are estimated both for black box optimisation, where the 2-opt neighbourhood operator is applied to an unknown TSP of the given size (10 and 100), and secondly for white box optimisation, where the values are estimated for a single TSP.

Starting with black box optimisation on  the class of 10 city TSP, 
We generated 20 random 10-location TSPs with edge length between 1 and 20.
Figure \ref{fig:tspct10} shows the fitness probabilities $p(k)$ for all values $k$ for which $p(k) > 0$.
In this set the optimum fitness was $66$, the modal fitness was $105$ and the good enough fitness, halfway between the modal and optimum fitness, was $85$.


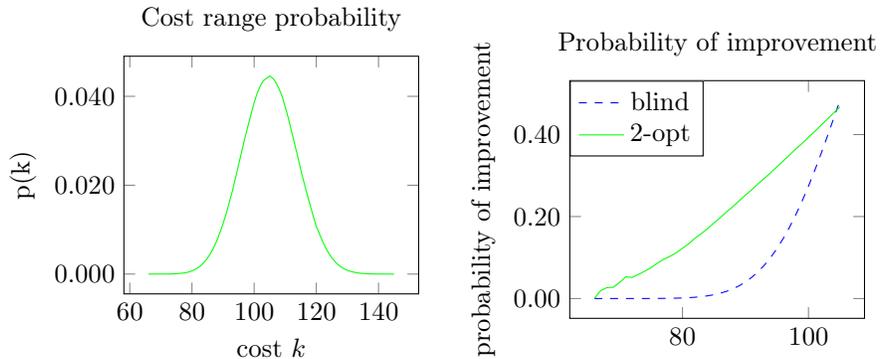
\begin{figure}[!ht]
\begin{tikzpicture}
\begin{axis}[
width=5.5cm,
title= Cost range probability,
scaled ticks = false,
xlabel={cost $k$},
ylabel={p(k)},
y tick label style={
                /pgf/number format/fixed,
                /pgf/number format/fixed zerofill,
                /pgf/number format/precision=3},
]

\addplot [green] table {tsp_multi_p1020201.dat};
\end{axis}
\end{tikzpicture}
\hskip 10pt
\pgfplotsset{
    every axis legend/.append style={
    at={(0,1)},
    anchor=north west,
                                    },
            }
\begin{tikzpicture}
\begin{axis}[
width=5.5cm,
title= Probability of improvement,
ylabel={probability of improvement},
y tick label style={
                /pgf/number format/fixed,
                /pgf/number format/fixed zerofill,
                /pgf/number format/precision=2},
legend entries={blind,2-opt},
]
\addplot [blue] [dashed] table [x=xx,y=yy] {tsp_multi_ppn1020201.dat};
\addplot [green] [solid]  table [x=xx,y=zz] {tsp_multi_ppn1020201.dat};
\end{axis}
\end{tikzpicture}

\caption{TSP 10 class, cost and improvement probabilities}
\label{fig:tspct10}
\end{figure}
Even for this small problem size the probabilities at each cost level decrease towards the optimum.

The next set of results compares the probability of improving from a cost of $k$, as it varies from $66$, which is the optimum, to $105$, which is 
the modal cost level.
We recorded the average probability of improvement for each cost level below $105$.
The dotted blue line is the probability using blind search (based on the proportion of points in the search space that have cost lower than $k$), and the green line using 2-opt neighbourhood search (based on the proportion of points in the neighbourhoods of points with a cost of $k$ that have cost lower than $k$).

The right hand graph in figure \ref{fig:tspct10} suggests that for the TSP with 2-swap neighbourhoods, in contrast to blind search, the proportion of improving neighbourhood moves decreases almost linearly as the cost of the current point approaches the optimum.
An intuition as to why this decrease is linear comes from a simple problem of minimising the sum of $n$ zero-one variables. Assume a neighbour results from flipping a single variable. 
When the current point has cost $k$, improvement comes from flipping any of $k$ out of $n$ variables, so the probability of improvement is $k/n$.  
Clearly as $k$ decreases, $k/n$ decreases linearly.

Turning to white-box optimisation, we generated a single 10 city TSP, with inter-city edge lengths are randomly generated in the range 1 to $20$, and generated the values for $p(k)$ and the probability of improvement under neighbourhood search and blind search as before.  
The optimum fitness was $69$ the modal fitness $96$ and the good enough fitness $82$.
The results are shown in figure \ref{fig:whitetspct10}.
\begin{figure}[!ht]
\begin{tikzpicture}
\begin{axis}[
width=5.5cm,
title= Cost range probability,
scaled ticks = false,
xlabel={cost $k$},
ylabel={p(k)},
y tick label style={
                /pgf/number format/fixed,
                /pgf/number format/fixed zerofill,
                /pgf/number format/precision=3},
]

\addplot [green] table {tsp_multi_p102011.dat};
\end{axis}
\end{tikzpicture}
\hskip 10pt
\pgfplotsset{
    every axis legend/.append style={
    at={(0,1)},
    anchor=north west,
                                    },
            }
\begin{tikzpicture}
\begin{axis}[
width=5.5cm,
title= Probability of improvement,
xlabel={starting cost},
ylabel={probability of improvement},
y tick label style={
                /pgf/number format/fixed,
                /pgf/number format/fixed zerofill,
                /pgf/number format/precision=2},
legend entries={blind,2-opt},
]
\addplot [blue] [dashed] table [x=xx,y=yy] {tsp_multi_ppn102011.dat};
\addplot [green] [solid]  table [x=xx,y=zz] {tsp_multi_ppn102011.dat};
\end{axis}
\end{tikzpicture}
\caption{TSP 10 instance, cost and improvement probabilities}
\label{fig:whitetspct10}
\end{figure}
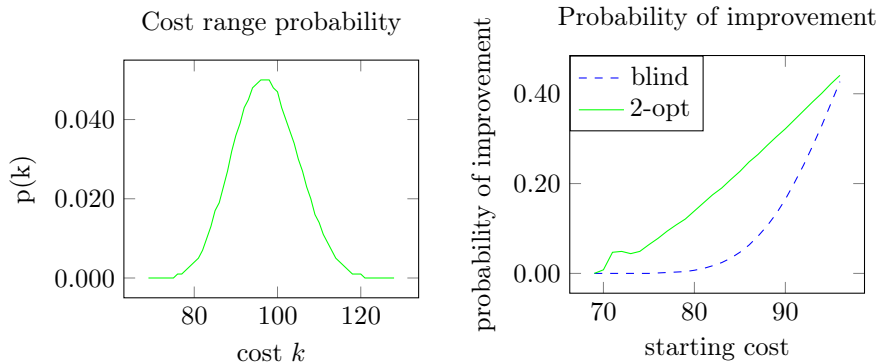

We next investigate the $NSF(k)$ properties
\begin{itemize}
    \item the NSF weight decreases with increasing fitness difference:\\
    $\forall \delta_1 < \delta_2 : r(k,\delta1) \geq r(k,\delta2)$
    \item the neighbourhood is normal:\\
    $\forall \delta : pn(k,k-\delta) \geq r(k,\delta) \times p(k-\delta)$
\end{itemize}
We test the case where $k=k_{ge}$, the furthest fitness from the optimum where we still expect $NSF(k)$ to hold.
For the black-box optimisation example the good enough fitness value $k_{ge}$ is $85$, so on the left hand side we show the neighborhood is normal by graphing $pn(85,85-\delta)$ and $r(85,\delta) \times p(85-\delta)$, which shows that $pn(85,85-\delta)$ takes a higher value than $r(85,\delta) \times p(85-\delta)$.  On the right hand side of figure \ref{fig:class10nsf} we show how the NSF weight $r(85,\delta)$ decreases for increasing values of $\delta$.

In short the 2-opt operator on the TSP 10 class indeed satisfies the $NSF(k)$ property for the value $k_{ge}= 85$.  We also tested $NSF(k)$ for all values from $67 \ldots 85$.
To test all neighbourhoods were normal, for each of the $90$ pairs, $k \in 67 \ldots 85, \delta \in 1 \ldots k-66$, we confirmed that $pn(k,k-\delta) \geq p(k-\delta) \times r(k,\delta)$.
Also for each $k \in 67 \ldots 85$ and $\delta \in 1 \ldots k-67$ we checked  whether $r(k,\delta) \geq r(k,\delta+1)$.
This NSF property help for all values of $k$ and $\delta$ except $k=70, \delta=3$.
We explore cases where the $NSF$ property is nearly, but not completely satisfied, in section \ref{sec:weakercond}.

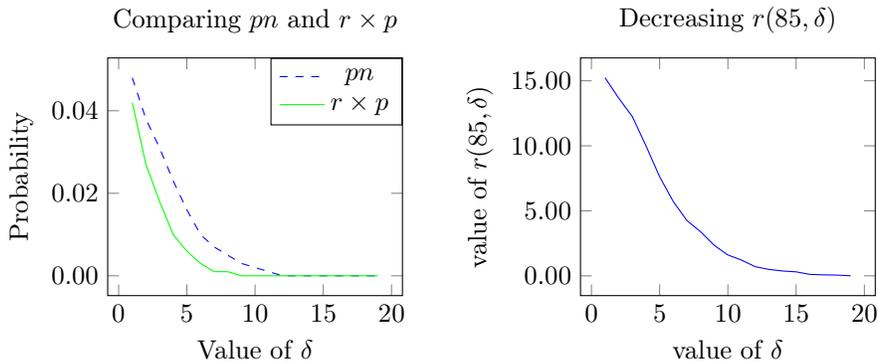
\begin{figure}[hbt]
\pgfplotsset{
    every axis legend/.append style={
    at={(1,1)},
    anchor=north east,
                                    },
            }
\begin{tikzpicture}
\begin{axis}[
width=5.5cm,
title= {Comparing $pn$ and $r \times p$},
scaled ticks = false,
xlabel={Value of $\delta$},
ylabel={Probability},
y tick label style={
                /pgf/number format/fixed,
                /pgf/number format/fixed zerofill,
                /pgf/number format/precision=2},
legend entries={$pn$,$r \times p$},
]
\addplot [blue] [dashed] table {tsp_multi_pn_kge102020.dat};
\addplot [green] [solid] table {tsp_multi_p_r_kge102020.dat};
\end{axis}
\end{tikzpicture}
\hspace{10pt}
\begin{tikzpicture}
\begin{axis}[
width=5.5cm,
title= {Decreasing $r(85,\delta)$},
scaled ticks = false,
xlabel={value of $\delta$},
ylabel={value of $r(85,\delta)$},
y tick label style={
                /pgf/number format/fixed,
                /pgf/number format/fixed zerofill,
                /pgf/number format/precision=2},
]
\addplot [blue] table {tsp_multi_r102020.dat};
\end{axis}
\end{tikzpicture}
\caption{Illustration of $NSF(85)$ for the TSP 10 class}
\label{fig:class10nsf}
\end{figure}



We next investigate white-box optimisation: generating a single (random) TSP 10 problem instance.  
We claim that typically for fitness values better than $k_{ge}$, the property $NSF(k)$ holds.
However, while the fitness probabilities do decrease towards the optimum, and while the NSF weight does decrease with increasing fitness difference, in particular for smaller problem instances, the decrease is not always strictly monotonic.
For example in the TSP 10 instance, where $k_{ge}=82$, the decrease in $r(82,\delta)$ with increasing $\delta$, as shown in figure \ref{fig:instance10nsf} is not strictly monotonic.

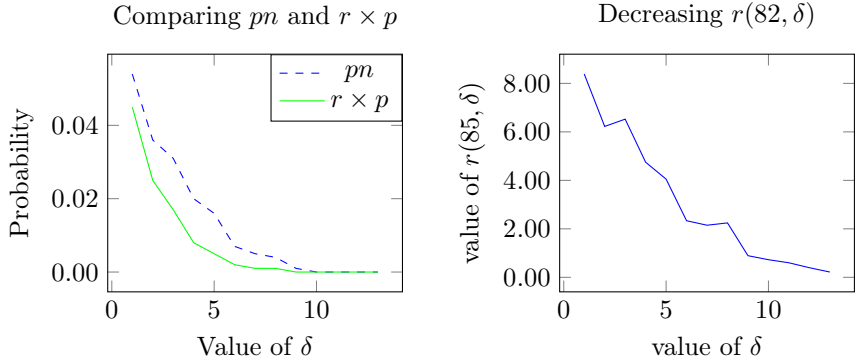
\begin{figure}[hbt]
\pgfplotsset{
    every axis legend/.append style={
    at={(1,1)},
    anchor=north east,
                                    },
            }
\begin{tikzpicture}
\begin{axis}[
width=5.5cm,
title= {Comparing $pn$ and $r \times p$},
scaled ticks = false,
xlabel={Value of $\delta$},
ylabel={Probability},
y tick label style={
                /pgf/number format/fixed,
                /pgf/number format/fixed zerofill,
                /pgf/number format/precision=2},
legend entries={$pn$,$r \times p$},
]
\addplot [blue] [dashed] table {tsp_multi_pn_kge10201.dat};
\addplot [green] [solid] table {tsp_multi_p_r_kge10201.dat};
\end{axis}
\end{tikzpicture}
\hspace{10pt}
\begin{tikzpicture}
\begin{axis}[
width=5.5cm,
title= {Decreasing $r(82,\delta)$},
scaled ticks = false,
xlabel={value of $\delta$},
ylabel={value of $r(85,\delta)$},
y tick label style={
                /pgf/number format/fixed,
                /pgf/number format/fixed zerofill,
                /pgf/number format/precision=2},
]
\addplot [blue] table {tsp_multi_r10201.dat};
\end{axis}
\end{tikzpicture}
\caption{Illustration of $NSF(82)$ for the TSP 10 instance}
\label{fig:instance10nsf}
\end{figure}
\subsection{Weaker conditions for beneficial neighbourhood search}
\label{sec:weakercond}
Nevertheless we now show that, in fact, the $NSF(k)$ property nearly holds, sufficient to support the result that neighbourhood search is beneficial.

To prove the result of theorem \ref{thm:improvement} (that neighbourhood search is beneficial), it is not necessary the the $NSF$ properties strictly hold.
Assuming normal neighbourhoods (as shown on the right hand side of figure \ref{fig:instance10nsf}), lemma 
\ref{lemma:improver} suffices to yield the required result.

To show that even for the TSP 10 instance, neighbourhood search robustly improves on blind search, 
we check the conditions of lemma 
\ref{lemma:improver}:
\begin{align*}
pbr^<(k) \geq \bar{r}(k) \times p^<(k) \\
\bar{r}(k)\geq 1
\end{align*}
which, with normal neighbourhoods, suffice to entail the desired result\\
$$pn^<(k) \geq p^<(k)$$
that neighbourhood search outperforms blind search.

First we calculate $\bar{r}(k) : k \in 1..k_{ge}$ for the TSP 10 class and the TSP 10 instance, shown in tables \ref{tab:tsp10cavr} and \ref{tab:tsp10iavr}
\begin{table}[htb]
\scriptsize 
\centering
    \begin{tabular}{| c | cccccccccccccc |}
    \toprule
$k$  & 67&68&69&70&71 & $\ldots$ & 78 & 79 & 80 & 81 & 82 & 83 & 84 & 85\\
$\bar{r}(k)$ &23441 &6771 &3600 &2066 &1217 & $\ldots$ & 41 & 28 & 19 & 14 & 10 & 7 & 5 & 4\\
    \bottomrule
  \end{tabular}
  \vspace{1ex}
  \caption{ Average NSF weight for the TSP 10 class}
  \label{tab:tsp10cavr}
\end{table}
\begin{table}[htb]
\scriptsize  
\centering
    \begin{tabular}{| c | c c c c c c c c c c c c c |}
    \toprule
$k$  & 70 & 71 & 72 & 73 & 74 & 75 & 76 & 77 & 78 & 79 & 80 & 81 & 82\\
$\bar{r}(k)$ & 1054 & 583 & 292 & 130 & 78 & 45 & 28 & 18 & 12 & 8 & 5.6 & 4.1 & 3.0\\
    \bottomrule
  \end{tabular}
  \vspace{1ex}
  \caption{ Average NSF weight for the TSP 10 instance}
  \label{tab:tsp10iavr}
\end{table}
For both the TSP 10 class and the instance the average NSF weight, $\bar{r}$, satisfies $\bar{r}(k) > 1$ for all values of $k$ between the optimum and the good enough fitness.



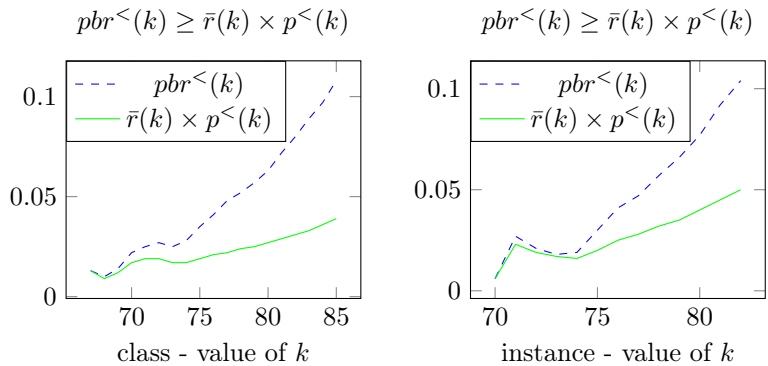
\begin{figure}[hbt]
\begin{tikzpicture}
\pgfplotsset{
    every axis legend/.append style={
    at={(0,1)},
    anchor=north west,
                                    },
            }
\begin{axis}[
width=5.5cm,
title= {$pbr^<(k) \geq \bar{r}(k) \times p^<(k)$},
scaled ticks = false,
xlabel={class - value of $k$},
yticklabel style={
        /pgf/number format/fixed,
        /pgf/number format/precision=2
},
legend entries={$pbr^<(k)$,$\bar{r}(k) \times p^<(k)$},
]
\addplot [blue] [dashed] table {tsp_multi_pbr_low102020.dat};
\addplot [green] [solid] table {tsp_multi_pavr_low102020.dat};
\end{axis}
\end{tikzpicture}
\hspace{10pt}
\begin{tikzpicture}
\pgfplotsset{
    every axis legend/.append style={
    at={(0,1)},
    anchor=north west,
                                    },
            }
\begin{axis}[
width=5.5cm,
title= {$pbr^<(k) \geq \bar{r}(k) \times p^<(k)$},
scaled ticks = false,
xlabel={instance - value of $k$},
yticklabel style={
        /pgf/number format/fixed,
        /pgf/number format/precision=2
},
legend entries={$pbr^<(k)$,$\bar{r}(k) \times p^<(k)$},
]
\addplot [blue] [dashed] table {tsp_multi_pbr_low10201.dat};
\addplot [green] [solid] table {tsp_multi_pavr_low10201.dat};
\end{axis}
\end{tikzpicture}
\caption{Conditions for lemma \ref{lemma:improver} in the TSP 10 class and instance}
\label{fig:pbrinstance}
\end{figure}

Tests on all values of $k$ confirm that 
$pbr^<(k) \geq \bar{r}(k) \times p^<(k)$ 
holds for all values of $k$ between the optimum and the good enough fitness for both the TSP 10 class and the TSP 10 instance.  The results for the TSP 10 instance are shown in
Figure \ref{fig:pbrinstance}. 
This provides evidence that the necessary conditions of theorem \ref{thm:improvement} hold strongly for all the TSP 10 tests, both for black box and white box optimisation.

\subsection{A larger TSP problem class}


To scale up to a larger problem size we generated random distances for 20 100-location TSP instances, with inter-city distances up to 50.  
We then sampled 2,000,000 points (100,000 per instance) and counted the number of points at each cost level.  
From this we estimated $p(k)$ for each cost $k$.
We found the minimum cost $2190$ and maximum cost $2920$ in this sample, and set a ``starting cost'' as $k_{ge} = 2371$.
We then checked 2,000,000 points (again 100,000 per instance) recording all those with cost $2371$.
We found all the neighbours of all these points and thus estimated $pn(2371,k)$ for all values of $k$.

As in figure \ref{fig:tspct10} above, we plotted the fitness probabilities and the probability of improvement for the TSP 100 class, in figure \ref{fig:tspct100}.

\begin{figure}[!ht]
\begin{tikzpicture}
\begin{axis}[
width=5.5cm,
title= Cost range probability,
scaled ticks = false,
xlabel={cost $k$},
ylabel={p(k)},
]

\addplot [green] table {tsp_approx_p100_50_20.dat};
\end{axis}
\end{tikzpicture}
\hskip 10pt
\pgfplotsset{
    every axis legend/.append style={
    at={(0,0.5)},
    anchor=north west,
                                    },
            }
\begin{tikzpicture}
\begin{axis}[
width=5.5cm,
title= Probability of improvement,
xlabel={starting cost},
ylabel={probability of improvement},
y tick label style={
                /pgf/number format/fixed,
                /pgf/number format/fixed zerofill,
                /pgf/number format/precision=2},
legend entries={2-opt,blind},
]
\addplot [green] [solid]  table [x=xx,y=zz] {tsp_approx_p_pn100_50_20.dat};
\addplot [blue] [dashed] table [x=xx,y=yy] {tsp_approx_p_pn100_50_20.dat};
\end{axis}
\end{tikzpicture}

\caption{TSP 100 class, cost and improvement probabilities}
\label{fig:tspct100}
\end{figure}
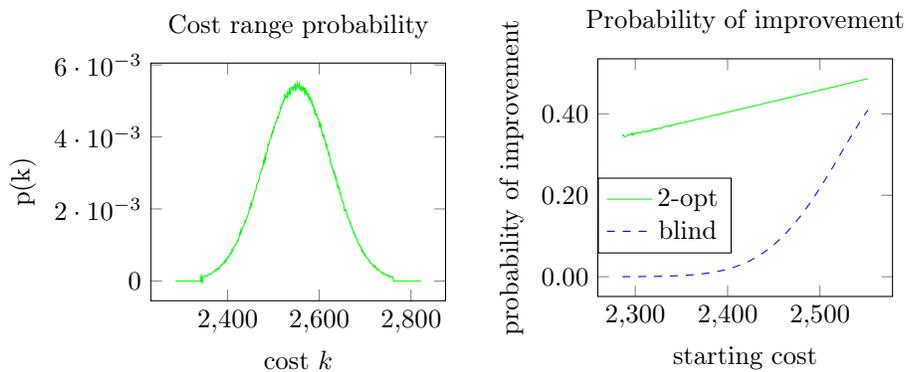

Naturally neighbourhood search does considerably better than blind search.
The $NSF(2371)$ properties are shown for the TSP 100 class in figure \ref{fig:class100nsf}.
This problem class, as a typical combinatorial problem, largely satisfies the $NSF$ properties.
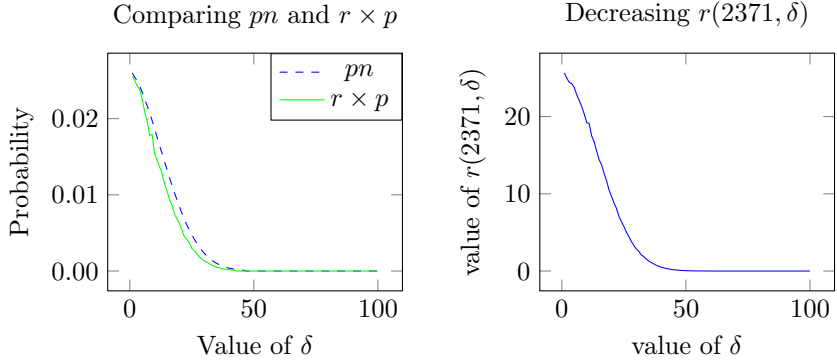
\begin{figure}[hbt]
\pgfplotsset{
    every axis legend/.append style={
    at={(1,1)},
    anchor=north east,
                                    },
            }
\begin{tikzpicture}
\begin{axis}[
width=5.5cm,
title= {Comparing $pn$ and $r \times p$},
scaled ticks = false,
xlabel={Value of $\delta$},
ylabel={Probability},
y tick label style={
                /pgf/number format/fixed,
                /pgf/number format/fixed zerofill,
                /pgf/number format/precision=2},
legend entries={$pn$,$r \times p$},
]
\addplot [blue] [dashed] table {tsp_approx_pn_kge100_50_20.dat};
\addplot [green] [solid] table {tsp_approx_p_r_kge100_50_20.dat};
\end{axis}
\end{tikzpicture}
\hspace{10pt}
\begin{tikzpicture}
\begin{axis}[
width=5.5cm,
title= {Decreasing $r(2371,\delta)$},
scaled ticks = false,
xlabel={value of $\delta$},
ylabel={value of $r(2371,\delta)$},
y tick label style={
                /pgf/number format/fixed,
                /pgf/number format/fixed zerofill,
                /pgf/number format/precision=0},
]
\addplot [blue] table {tsp_approx_r_kge100_50_20.dat};
\end{axis}
\end{tikzpicture}
\caption{Illustration of $NSF(2371)$ for the TSP 100 class}
\label{fig:class100nsf}
\end{figure}

As the starting fitness $k$ improves, the average NSF weight $\bar{r}(k)$ grows dramatically.
This is the reason for the phenomenon shown on the left hand side of figure \ref{fig:tspct100}: the probability of improvement with local search is so much greater than with blind search.
To end this section, in figure \ref{fig:class100nsfwt} we plot $\bar{r}(k)$ against $k$ in the TSP 100 class and a TSP 100 instance.
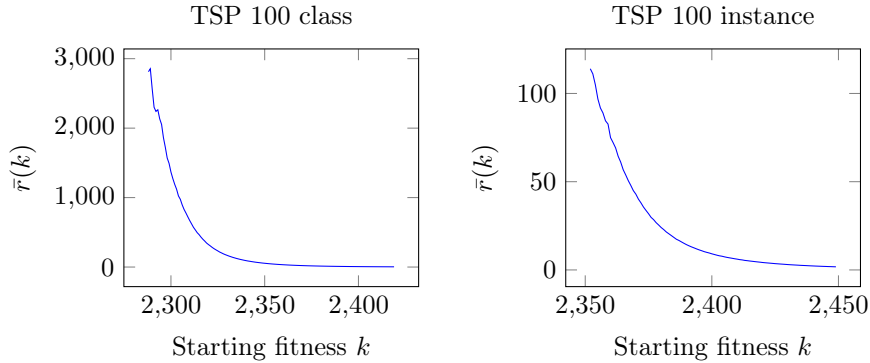
\begin{figure}[hbt]
\begin{tikzpicture}
\begin{axis}[
width=5.5cm,
title= {TSP 100 class},
scaled ticks = false,
xlabel={Starting fitness $k$},
ylabel={$\bar{r}(k)$},
y tick label style={
                /pgf/number format/fixed,
                /pgf/number format/fixed zerofill,
                /pgf/number format/precision=0},
]
\addplot [blue] table {tsp_approx_avr100_50_20.dat};
\end{axis}
\end{tikzpicture}
\hspace{10pt}
\begin{tikzpicture}
\begin{axis}[
width=5.5cm,
title= {TSP 100 instance},
scaled ticks = false,
xlabel={Starting fitness $k$},
ylabel={$\bar{r}(k)$},
y tick label style={
                /pgf/number format/fixed,
                /pgf/number format/fixed zerofill,
                /pgf/number format/precision=0},
]
\addplot [blue] table {tsp_approx_avr100_50_1.dat};
\end{axis}
\end{tikzpicture}
\caption{Average NSF weights for TSP 100}
\label{fig:class100nsfwt}
\end{figure}

\section{Expected Rate of Improvement}

\subsection{Expected Improvement from a single step}
\label{sec:betterblind}
By theorem \ref{thm:improvement} the $NSF(k)$ property entails that, 
given a good enough starting fitness $k$ and a limit on the probability of neighbours having the same fitness, 
search in the neighbourhood has a higher probability of improvement than blind search.  
However blind search can yield a point with \emph{much} better fitness, albeit with lower probability.
In fact the expected improvement from blind search can actually be greater.

To define the expected improvement from a single step, we consider the probability $pn(k,k')$ of picking a point of cost $k'$.  If $k'$ is lower than $k$, this yields an improvement of $k-k'$.
If, on the other hand, $k'$ is higher than $k$, the neighbour with cost $k'$ is ignored, and there is no ``negative improvement''. 
In short the improvement is $0$. 

The expected improvement from one step $en_{imp}(k)$ is therefore:
\begin{equation}
en_{imp}(k) = \sum_{k'<k} pn(k,k') \times (k-k')
\label{eq:enimp}
\end{equation}
Instead of searching in the neighbourhood, the system could use blind search to try to improve on $k$.  In this case the expected improvement $e_{imp}(k)$ is defined similarly:
\begin{equation}
e_{imp}(k) = \sum_{k'<k} p(k') \times (k-k')
\label{eq:eimp}
\end{equation}
To illustrate this we take the first benchmark problem class above with 200 cost levels, and the same cost probability at all levels.
We use the same NSF weightings, with maximum fitness difference $b$ between neighbours of $b=1,b=5,b=10,b=50,b=200$, and start at cost $30$.
For all these values of $b$, neighbourhood search is beneficial.
However, table \ref{tab:betterblind} shows the values  for $e_{imp}(30)$ in the first column and $en_{imp}(30)$ for different values of $b$ in the remaining columns:

\begin{table}[!ht]
\centering
    \begin{tabular}{| c | c | c c c c c |}
    \toprule
 & $e_{imp}(30)$ & \multicolumn{5}{c |}{$en_{imp}(30)$}\\
& & $b=1$ & $b=5$ &  $b=10$ & $b=50$ & $b=200$ \\ 
    \midrule
$k=30$ & 2.32 & 0.33 & 1.36 & 2.61 & 5.81 & 2.32\\
    \bottomrule
  \end{tabular}
  \vspace{1ex}
  \caption{ Comparing $e_{imp}(30)$ and $en_{imp}(30)$
  }
  \label{tab:betterblind}
\end{table}
From a starting cost of $30$, the expected improvement from neighbourhood search only exceeds that from blind search when $b \geq 9$.

In the rest of this section  we address the question: what happens after a sequence of search steps?
If the maximum fitness difference between neighbours is less than 9 ($b < 9$) does blind search continue to improve faster than local descent?
Or does local descent always outperform blind search, given enough steps?

\subsection{Local Descent}




\subsubsection{The number of steps to improve \emph{imp(k)}}

If the current point has cost $k$, we determine $imp(k)$ the expected number of steps to improve (i.e. move to a point with better cost).
The probability that after one step local descent can move to a better point is simply $pn^<(k)$.
Thus the probability of improving after two steps is $(1-pn^<(k)) \times pn^<(k)$. 
Generalising, the probability of improving after $n$ steps is:
$$ (1 - pn^<(k))^{n-1} \times pn^<(k) $$

A local descent algorithm on a single problem instance will have some limit $N$ on the number of attempts to find an improving neighbour of the current point, or the current plateau.
If the local descent algorithm on problem classes were restricted in the same way to a limit of $N$ attempts to improve, the expected number of steps to improve would be: 
$$ \sum_{j =0}^{N-1} : (1+j) \times (1 - pn^<(k))^j \times pn^<(k) $$
assuming that an improving neighbour was found within $N$ attempts.

Instead we place no limit on the number of attempts to improve in local descent on problem classes.
Assuming $pn^<(k) \geq p^<(k)$,
$pn^<(k)$ is strictly positive for all cost values except the global optimum.
In this case, the probability $(1-\sum_{i<k} pn(k,i))^{N}$ that no improving neighbour is found tends to $0$ as $N$ tends to infinity.

Accordingly the expected number of steps to improve is:
$$ \sum_{j =0}^{\infty} : (1+j) \times (1 - pn^<(k))^j \times pn^<(k) = 1/pn^<(k)$$
This is an upper bound on the expected number of steps required by a local descent algorithm on any problem instances when it is not at a local optimum.
Simplifying, yields the following equation for $imp(k)$, the expected number of steps to improve from cost $k$:
\begin{equation}
    imp(k) = 1 / pn^<(k)
    \label{imp-repl}
\end{equation}
The definition of $imp(k)$ in equation (\ref{imp-repl}) could be interpreted as a tight bound on the number of steps required by concrete local descent if the neighbourhoods contained a very large number of points.

\subsubsection{Formalising local descent}
\label{sec:formalising}
The function $steps(k)$ encodes the expected number of steps, starting with at a point $s1$ with cost $k$ to reach a point $s2$ with a cost of $0$, assuming $s2$ is reached from $s1$ by {\em local descent}.

Since the function uses $imp(k)$ for the expected number of steps to improve, this function is an overestimate of the expected number of steps for concrete local descent to reach any improving fitness $k' < k$.

\begin{definition}[Local descent] 
$steps(k)$ is specified by the following recursive definition, where $k$ is the starting cost:

\[
steps(k) = \left.
  \begin{cases}
  0 & \text{if} \  k = 0 \\
  imp(k) + \sum_{i<k} \frac{pn(k,i)}{ pn^<(k) } \times steps(i) & \text{if} \ k > 0
    \end{cases}
    \right\}
\]
\label{def:stepsdef}
\end{definition}

Clearly if the current cost is already $0$, $0$ steps are needed.

The second expression, for $k>0$, models a search which starts at a point with cost $k$, and chooses neighbours of points with that cost until a better point is found (the number of steps being $imp(k)$).
Weighted by the probability that the next point has cost $i$, $steps(i)$ calculates the remaining steps.

Blind search arbitrarily selects a point in the search space, returns its cost, and then tries again.
In this case the expected number of steps $blind$ for blind search to find an optimal point is:\\
\begin{equation*}
    blind = \sum_{i=1}^{\infty} i \times ((1-p(0))^{i-1} \times p(0)) = 1/p(0)
\end{equation*}

\begin{definition}[Beneficial local descent]
We say local descent starting at a point with cost $k$ is beneficial if 
$steps(k) \leq blind$\footnote{In fact this result goes through if instead of the optimum, $0$, the target is $k'$ where $p^<(k')$ is no greater than any other fitness probability greater than $k'$.  This is discussed just before theorem \ref{thm:localdescent} below.}
\end{definition}

\subsection{Full NSF for local descent}
\label{sec:nsf2}
For beneficial improvement, in the previous section, the $NSF(k)$ property only needed to hold for the current cost $k$.
For local descent, however, the current cost changes whenever improvement occurs.
Thus $NSF(k)$ must hold at all levels better than the starting cost of the local descent.

A simple way in which to ensure $NSF(k)$ holds for fitness levels better than a starting level $k_0$
is to enforce $NSF(k_0)$, and have $r(k,i) = r(k_0,i)$ for all $i \leq k \leq k_0$.
Thus in the computational experiments below, we input a sequence of values for the NSF weights that are independent of the current fitness level.  

However experiments on combinatorial problems show that as the fitness $k$ approaches the optimum, $r(k,i)$ actually increases with decreasing values of $k$.
Indeed for local descent to be beneficial, $r(k,\delta)$ should not decrease as $k$ improves.
Suppose a local descent finds a sequence of points with better and better fitness $k_1, k_2, k_3, \ldots$.
If, at some $k_j$ in this sequence, $r(k_j,1) = 1$, then 
all the steps until $k_j$ have been worthless. 
The search from $k_j$ now has the same expected number of steps as blind search,
and therefore the local descent uses more steps than blind search.

To preclude this, and other related local descents, we require the following property:

\begin{definition}
Full NSF holds from level $k_0$ if:
\begin{flalign*}
\forall k \leq k_0 : & NSF(k)&\\
\forall i \leq k_2 < k_1 \leq k_0 : & r(k_1,i) \leq r(k_2,i)&  
\end{flalign*}
\label{def:fullnsf}
\end{definition}
We illustrate that full NSF typically holds from $k_{ge}$, first with a benchmark problem class, and then with the randomly generated TSP.

First we examine the linear benchmark introduced in section \ref{sec:bench} above, where ($p(i) = 0.5 - \rm{abs}(100-i)/200$ for each cost level $i$).
In this benchmark, the fixed bound $b$ is the maximum fitness difference between any neighbours. 
For any fitness difference $\delta > b$, the NSF weight $r(k,\delta) = 0$.
For smaller fitness difference the NSF weight $r(k,\delta)$ is independent of $\delta$.

We fixed $b=10$ and for each starting cost $k \in 1..100$ we calculated
$r(k,1)$. 
The graph on the left hand side of figure~\ref{fig:rki} shows how $r(k,1)$ increases as $k$ decreases towards the optimum.
Since in this benchmark $r(k,\delta)= r(k,1)$ for all $\delta \leq 10$,
and $r(k,\delta) = 0$ for all $\delta > 10$,
we have shown that $r(k,\delta)$ increases with decreasing $k$ for all values of $\delta$.

\begin{figure}[hbt]
\pgfplotsset{
    every axis legend/.append style={
    at={(1,1)},
    anchor=north east,
                                    },
            }
\begin{tikzpicture}
\begin{axis}[
width=5.5cm,
title= {Linear benchmark: $r(k,1)$},
xlabel={Starting cost: $k$},
ylabel={NSF weight: $r(k,\delta)$},
legend entries={b=10},
]
\addplot [blue] [solid]  table {rkiklinear10.dat};
\end{axis}
\end{tikzpicture}
\hspace{10pt}
\begin{tikzpicture}
\begin{axis}[
width=5.5cm,
title= {TSP 100 - $r(k,\delta)$},
xlabel={Cost change: $\delta$},
ylabel={NSF weight: $r(k,\delta)$},
legend entries={{k=2258},{k=2248},{k=2238},{k=2228},{k=2218}},
]
\addplot [blue, solid, line width = 2.5pt] table {mytspr10010225.dat};
\addplot [blue, solid, line width = 2.0pt]  table {mytspr10010224.dat};
\addplot [blue, solid, line width = 1.5pt ] table {mytspr10010223.dat};
\addplot [blue, solid, line width = 1.0pt ] table {mytspr10010222.dat};
\addplot [blue, solid, line width = 0.5pt ] table {mytspr10010221.dat};
\end{axis}
\end{tikzpicture}
\caption{NSF weights}
\label{fig:rki}
\end{figure}
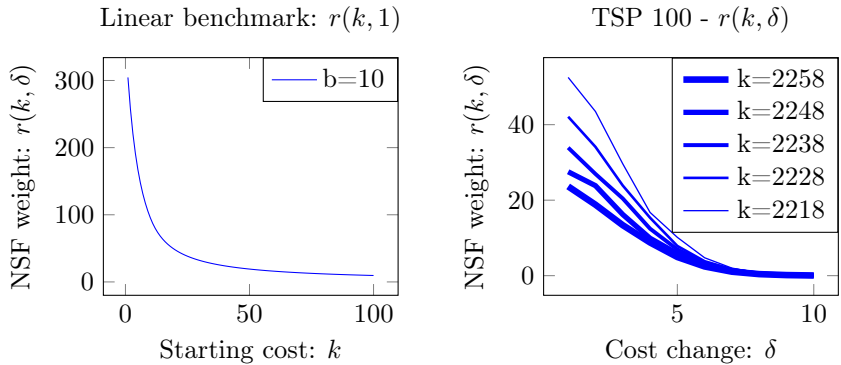

The graph on the right hand side of figure \ref{fig:rki} shows values of $r(k,\delta)$ in a 100-location TSP. 
In this case instead of the x-axis showing values of the starting cost $k$, the x-axis shows increasing values of $\delta$, and different values of $k$ are shown as different lines.
Again higher values of $k$ yield lower values of $r(k,\delta)$ for all values of $\delta$.
In this graph there is a different curve for each value of $k$, showing how larger values of $k$ yield smaller values for $r(k,\delta)$.
The neighbourhood is 2-swap, and the data is based on all the neighbours of 20 points.
The horizontal axis shows how $r(k,\delta)$ also decreases with increasing $\delta$ (confirming the NSF property).
The graph shows $r(k,\delta) : \delta \in 1..10$, for 5 different values of $k$: 
$2258, 2248,2238,2228,2218$.

\subsection{Software Experiments}
Given the value of $pn(i,j)$ for each $k \geq i \geq j$, we can evaluate $steps(k)$.

To explore the value of $steps(k)$ for different problem classes, we wrote a program with
the following parameters:

\begin{itemize}
    \item[Size] The number of cost levels from the optimum up to the modal cost: $n$
    \item[Counts] The number of points at each cost level
    \item[Total] The total number of points in the search space: $total$
    \item[Weights] The NSF weights: $r(\_,i)$ for $i \in 1, \ldots k_{ge}$ (assuming the value of $r(x,i)$ is independent of $x$, and $r(x,0)=1$) 
\end{itemize}
The cost probabilities $p(k)$ were inferred as $Count/Total$
We generated a large variety of problem classes and we calculated the expected number of steps in each case.

For example we generated a problem with 1 optimum point, and five times as many points at each subsequent cost level up to level 9, and a total of 500,000 points.
We then evaluated $steps(5)$.
We then explored how different NSF weights resulted in different expected number of steps to reach the optimum.

\begin{table}[!ht]
\centering
    \begin{tabular}{| c c c c c | c | c |}
    \toprule
      & \multicolumn{3}{c}{Problem class} & & Weights & $steps(k)$\\
\midrule
& $k$ & Counts & Total & & [1,1,1,1] & 500,000\\
& & & & & [2,1,1,1] & 274,937\\
&  5 & [1,5,25,125,625,3125 & 500,000 & & [4,1,1,1] & 145,741\\
& & 15625,78125,390625] &  & & [4,3.5,3,0.5] & 128,801\\
\bottomrule
  \end{tabular}
  \vspace{1ex}
  \caption{ Evaluating $steps(k)$ on an example problem }
  \label{tab:stepexperiments}
\end{table}
Blind search has an expected number of $500,000$ steps, which is more that $steps(k)$ in each case, except Weights = [1,1,1,1] which has, of course, the same expected number of steps.

In a later section we will observe that this model has some surprising results when $\bar{r}(k) < 1$ and even when the fitness probabilities are not decreasing towards the optimum.

\paragraph{A Constraint Program to prove NSF ensures local search is effective}
To explore the question whether $blind$ is always an upper bound on $steps(k)$, we implemented this same function using a constraint program (written in ECLiPSe \cite{eclipse}).  For a wide variety of problem classes we supplied values for Size, Counts, and Total and imposed the constraint that $blind \leq steps(k)$.
We enforced the constraint that the NSF weights were non-increasing to meet the NSF requirement \ref{eq:nsf}, and we set $pn(k,k)=p(k)$.
The program is presented in \ref{sec:eclipseprog}.

For every such problem class, we ran $steps(k_{ge})$ and the program returned the constraint that the NSF weights must all be very close to $1$. 
This means that, for each problem class tested, the only way to satisfy the constraint that $blind \leq steps(k)$ is if $pn(k_1,k_2)=p(k_2)$, and neighbourhood search, like blind search, selects points randomly from the search space.
For any other NSF weights, where $r(k,1)>1$ in other words, $steps(k) < blind$.

The challenge remained to prove that the result holds for all problem classes and all neighbourhoods satisfying the NSF property.

\subsection{Theorem on the benefit of local descent}
\label{sec:abstract-local-descent}
 
We prove that, from a point with good enough cost $k$, the expected number of steps to reach an optimum point $steps(k)$ is lower than with blind search.
\cite{cohen2020steepest} exhibited a problem with binary variables with a simple objective, that requires an exponential number of flips to reach a local optimum - even choosing the best neighbour each time.
Nevertheless we prove that if local descent continues to improve long enough, its expected number of steps - although only choosing the first improving neighbour - will be less than with blind search.

It is worth noting that in principle an initial point with cost $k$ must be found before starting the local descent from this cost.
Section \ref{sec:blindthenlocal} reveals that using blind search to seed local descent cannot cause local descent to be worse than blind search, but can in fact improve it.

The first step towards this proof considers only problem classes with uniform cost probability at all levels and normal neighbourhoods.
Actually uniform cost probability and normal neighbourhoods are both worst case scenarios.

We write $steps_u(k)$ for $steps(k)$ under the conditions that\\
$ \forall k_2 \leq k :  p(k_2)=p(0)=p$\\ \ and\\
$ \forall k2 < k1 \leq k : pn(k1,k2) = r(k1,k1-k2) \times p(k2)$\\
\begin{definition}
\[
steps_u(k) = \left.
  \begin{cases}
  0 & \text{if} \  k = 0 \\
  \frac{1}{pbr^<(k)} + \sum_{i<k} \frac{r(k,i) \times p}{ pbr^<(k)} \times steps_u(i) & \text{if} \ k > 0
    \end{cases}
    \right\}
\]
\end{definition}

\begin{lemma} \label{lemma:localdescent}
Assume $\bar{r}(k) \geq 1$ and full NSF holds from level $k$, 
then:
\begin{flalign*}
\ \ \ steps_u(k)& \leq blind&
\end{flalign*}
\end{lemma}
The proof is in \ref{sec:localdescentproof}

This lemma is then generalised for all problem classes, and all neighbourhoods that are normal.  

Returning to the software experiments, it emerged that the requirement for fitness probabilities to decrease towards the optimum is not necessary for local descent.
The only required condition is that the fitness probability of the optimum solution is as small as any other fitness level.
An example is in table \ref{tab:nogoodenough}.
\begin{table}[!ht]
\scriptsize  
\centering
    \begin{tabular}{| c c c c c | c | c |}
    \toprule
      & \multicolumn{3}{c}{Problem class} & & Weights & $steps(k)$\\
\midrule
& $k$ & Counts & Total & & [1,1,1,1] & 100\\
& & & & & [2,1,1,1] & 62.7\\
&  5 & [1,2,3,4,5,6,7,8,1] & 100 & & [3,2,1,0.7] & 42.7\\
& &  &  & & [3,2.5,1.5,0.5] & 40\\
\bottomrule
  \end{tabular}
  \vspace{1ex}
  \caption{ Evaluating $steps(k)$ on non-decreasing fitness probabilities }
  \label{tab:nogoodenough}
\end{table}

Indeed the proof carries through, requiring only that
$\forall k_2 \leq k : p(k_2) \geq p(0)$.

This result means that if local descent reaches a level of fitness $k'$, where $p^<(k')$ is as low as any fitness probability, the result goes through.
It follows by assigning fitness $0$ to all points with original fitness of $k'$ or less, and $k-k'$ to points with any other fitness $k$.
These are the conditions under which local descent is expected to outperform blind search.
\begin{theorem}
\label{thm:localdescent}
For any $k$ where:\\
$\bar{r}(k) \geq 1$\\
$\forall k_2 \leq k : p(k_2) \geq p(0)$\\
and full NSF holds from level $k$\\
it follows that
\begin{flalign*}
\ \ \ steps(k)& \leq blind&
\end{flalign*}
\end{theorem}
The proof is in \ref{sec:localdescentproof}.

Before concluding this section we note a final software experiment.
In this experiment the average NSF weight is less than one, violating the first condition of theorem \ref{thm:localdescent}.
The result is shown in table \ref{tab:lowavr}
\begin{table}[!ht]
\scriptsize  
    \centering
    \begin{tabular}{|ccc|c|c|}
    \hline
         $k$ & Counts & Total & Weights & $steps(k)$ \\ 
    \hline 
          5 & [1,10,10,10,20,20,10,10,1] & 100 & [1.1,1,1,0.8] & 92.3\\
          5 & [1,1,10,10,10,10,10,10,1] & 70 & [1.4, 1, 1, 0.5] & 59.7\\
          3 & [1,10,10,10,1] & 35 & [1.1,0.8] & 32.6\\
    \hline
    \end{tabular}
    \caption{Evaluating $steps(k)$ when $\bar{r}(k) < 1$}
    \label{tab:lowavr}
\end{table}

Since lemma \ref{lemma:localdescent} is key to the proof of theorem \ref{thm:localdescent}, and if the fitness probabilities are the same for all fitness values the NSF property entails $\bar{r}(k) \geq 1$, we have not found a proof that extends to this case.


\subsection{Benchmarks for the number of steps}
This section shows how the improvement in the expected number of steps, due to local descent, grows with the steepness with which the cost probability falls.
The benchmarks are those introduced in section \ref{sec:bench}: one with a uniform cost probability, two with a linearly decreasing cost probability towards the optimum cost, and a final one with a exponential decrease in cost probability towards the optimum.
In each benchmark $k_{opt}=0$ and $k_{max}=200$.
The two linear benchmarks have different slopes.
For $j \leq 100$, the first has $p(j)=j/10,000$ and the second has $p(j) = 5 \times j/50,000$.
The exponential benchmark has $p(j) = {200 \choose j+1}/2^{200}$.
In each case the expected number of steps for blind search is $\frac{1}{p(0)}$, 
which is $200, 10000, 50000, 2^{200}$ for the uniform, linear, steep linear and exponential problem classes respectively. 

The NSF weights reflect neighbourhoods with a fixed maximum fitness difference, $b$, between neighbours, as in section \ref{sec:bench} above.  In these benchmarks we use $b=5$.  Below this difference the cost weight, $r(k,\delta): \delta<b$, is constant.

Table (\ref{tab:impbench2}) shows the expected number of steps to reach the optimum using local descent starting at cost levels $k=50, k=30, k=10$ and $k=2$.

\begin{table}[htbp]
\scriptsize  
\centering
    \begin{tabular}{| c || c | c | c | c | c |}
    \toprule
                           & $blind$ & \multicolumn{4}{c |}{$steps(k)$}\\     
                           & & $k=50$ & $k=30$ & $k=10$ & $k=2$ \\ 
    \midrule
uniform                & 200 & 41.9 & 27.3 & 12.5 & 7.5 \\
linear                & 10000 & 70.7 & 54.7 & 36.6 & 25.5 \\
steep linear       & 50000 & 158 & 142 & 123 & 112 \\
exponential         & $1.6 \times 10^{60}$ & $5.9 \times 10^{10}$  & 
$5.9 \times 10^{10}$  & $5.9 \times 10^{10}$ & $4.3 \times 10^{10}$ \\
    \bottomrule
  \end{tabular}
  \vspace{1ex}
  \caption{Steps due to local descent, with decreasing starting cost
  }
  \label{tab:impbench2}
\end{table}
For local descent a difference emerges between the two linear problems - the benefit of local descent over blind search is increased by the steepness of the slope.
Curiously on the exponential problem, the number of steps to reach the optimal solution for local descent is almost independent of the starting cost.   
There is a slight decrease as the starting cost approaches the optimum, but the number of final steps needed to actually reach the optimum dominates the total.

The next set of results explore the impact of the tightness of NSF on the number of steps.  To increase tightness we reduce the maximum fitness difference, $b$, between neighbours.  
As before we assume $r(k,\delta): \delta<b$ is constant.
We report the expected number of steps $steps(50)$ for three problems, uniform, linear, exponential, introduced above, and for neighbourhoods with maximum cost distance between neighbours of $1,5,10,50$ and $200$.
The results are shown in table \ref{tab:stepsbench}.
For the exponential problem, the tighter the NSF the fewer staps are needed.  However for the other problems, the minimum number of steps is reached at different intermediate tightnesses between $1$ and $200$.
\begin{table}[htb]
\scriptsize  
\centering
    \begin{tabular}{| c || c | c | c | c | c | c | }
    \toprule
 & $blind$ &   \multicolumn{5}{c |}{steps(50)}\\     
 &  & $b=1$ & $b=5$ &  $b=10$ & $b=50$ & $b=200$ \\ 
    \midrule
uniform   & 200 &  150 & 41.9 & 29.9 & 55.5 & 200\\
linear & 10000 & 164 & 70.7 & 36.6 & 1400 & 10000 \\
steep linear & 50000 & 176 & 158 & 373 & 6728 & 50000 \\
exponential  & $1.6 \times 10^{60}$ & 39,140 & $5.9 \times 10^{10}$ & $2.5 \times 10^{17}$ & $1.4 \times 10^{48}$ & $1.6 \times 10^{60}$ \\
    \bottomrule
  \end{tabular}
  \vspace{1ex}
  \caption{ Expected number of steps to reach optimality}
  \label{tab:stepsbench}
\end{table}

\section{Seeding local descent with blind search}
\label{sec:blindthenlocal}
In the previous section the performance of blind search was compared with the performance of local descent starting from a point with good enough cost $k$.
If blind search is used to find this initial point, this increases the number of steps required for local descent.
We first show that even including an initial blind search to find a point with good enough cost, theorem \ref{thm:localdescent} still holds.

The number of steps $blsteps(k)$ to reach a starting cost of at most $k$,  and from this starting cost to reach the optimum, is the sum of three components:
\begin{enumerate}
    \item Blind search returns the optimum immediately
    \item Blind search immediately returns a point with cost $i$ which is at least as good as $k$, and uses local descent starting at cost $i$
    \item The first blind selection finds a point with cost worse than $k$, and the initial blind search resumes
\end{enumerate}

Reflecting these three alternatives,
the expected number of steps $blsteps(k)$ required for local descent from a cost $k$ seeded by blind search is:
\begin{flalign*}
     blsteps(k) = & \ (p(0) \times 1) &\\
     & + \sum_{i=1}^{k-1} p(i) \times (1+steps(i))& \\
     & + (1-\sum_{i=0}^{k-1} p(i)) \times (1 + blsteps(k))& \\
\end{flalign*}

The following theorem is proven in \ref{sec:blsearch}:
\begin{theorem}
\label{thm:blindthenlocal}
For any $k$ where \\
$\bar{r}(k) \geq 1$ \\
$\forall k_2 \leq k: p(k_2) \geq p(opt)$ \\
Full NSF from level $k$\\
$\forall k_2 \leq k, \delta : pn(k_2,k_2-\delta) \geq r(k_2,\delta) \times p(k_2-\delta)$ (normal)\\ 
$$blsteps(k) \leq blind$$
\end{theorem}

Accordingly, taking into account blind search steps required to start the local descent cannot make local descent any worse than blind search.
Surprisingly it can actually reduce the number of steps.
Let us return to a comparison of the expected improvement in one step from neighbourhood search and blind search (section \ref{sec:betterblind}).

We take the uniform benchmark problem above, with 200 cost levels, and uniform cost probability.
$e_{imp}(k)$ is the expected amount of improvement from blind search starting at a point with cost $k$ (equation \ref{eq:eimp} above), and $en_{imp}(k)$ is the expected amount of improvement from neighbourhood search (equation \ref{eq:enimp}).

Given a fixed value of $b=5$, table \ref{tab:kblind} shows for each starting cost $k=50, k=30, k=20, k=10, k=2$, the expected improvement in one step from neighbourhood search and blind search:

\begin{table}[htbp]
\scriptsize  
\centering
    \begin{tabular}{| c || c | c | c | c | c |}
    \toprule
                           & $k=50$ & $k=30$ & $k=20$ & $k=10$ & $k=2$ \\ 
    \midrule
$e_{imp}(k)$            & 6.37 & 2.32 & 1.05 & 0.275 & 0.015\\
$en_{imp}(k)$           & 1.36 & 1.36 & 1.36 & 1.36 & 0.27\\
    \bottomrule
  \end{tabular}
  \vspace{1ex}
  \caption{ Comparing $e_{imp}(k)$ and $en_{imp}(k)$
  }
  \label{tab:kblind}
\end{table}
At $k=50$ and $k=30$ the expected improvement in one step from blind search is higher than that from local descent.
In fact this holds for all values of $k \geq23$.

The following table (\ref{tab:blindstart}) shows the performance of starting with blind search and switching to neighbourhood search after finding a point with cost better than $k$.
In table \ref{tab:blindstart} all four problems are evaluated, with $b=5$.
\begin{table}[htb]
\scriptsize  
\centering
    \begin{tabular}{| c || c | c | c | c | c | c |}
    \toprule
 & $blind$ & \multicolumn{5}{c |}{blind+steps}\\     
& & $k=50$ & $k=30$ & $k=20$ & $k=10$ & $k=5$ \\ 
    \midrule
uniform & 200  & 27.4 & 22.6 & 21.9 & 26.8 & 39.8 \\
linear & 10100 & 65.2 & 67.7 & 87.7 & 213 & 657 \\
steep linear & 50500 & 152 & 155 & 175 & 302 & 778 \\
exponential  & $1.5 \times 10^{60}$ & $6.9 \times 10^{12}$ & $1.8 \times 10^{25}$ & $7.6 \times 10^{33}$ & $1.2 \times 10^{45}$ & $2.3 \times 10^{52}$ \\
    \bottomrule
  \end{tabular}
  \vspace{1ex}
  \caption{ Expected number of steps to reach optimality}
  \label{tab:blindstart}
\end{table}
For the $p(j)=1/200$ (uniform) benchmark, the smallest expected number of steps results from switching at $k=20$.  Indeed the minimum occurs when $k=23$.
For the $p(j)=j/10000$ (linear) benchmark the minimum occurs when $k=43$.

For those values of $k$ for which one step improvement with blind search exceeds that of local descent, the lowest value of $k$ is that for which $e_{imp}(k)-en_{imp}(k)$ is smallest.
This value of $k \leq 50$ for each benchmark is:
\begin{alignat*}{1}
uniform &\rightarrow k=23\\ 
linear &\rightarrow k=43\\ 
steep \ linear &\rightarrow k=50\\
exponential &\rightarrow k=50 
\end{alignat*}
It appears, but remains to be proven, that the number of steps is minimised if the switch happens for that value of $k$ for which $e_{imp}(k)-en_{imp}(k)$ is positive and minimised.

\section{The impact of this paper}
This paper justifies the intuition that using neighbourhoods in which neighbours have similar fitness is enough to escape the conditions of the no free lunch theorems \cite{nofreelunch}.
By neighbourhood search we mean selecting a neighbour randomly and evaluating its fitness.  

The ``neighbourhood similar fitness'' (NSF(k)) property, made precise in definition \ref{sec:nsf}, is not only intuitive, but also sufficient to support the proofs that neighbourhood search has a better chance of improving from the current level of fitness than blind search.
Neighbourhood search is proven to be beneficial at any fitness around which the decrease in fitness probability towards the optimum is monotonic. 
Counter-examples in the appendices show that violating each of these conditions can lead to blind search outperforming neighbourhood search.

Not every problem instance, and neighbourhood used for tackling it, {\em completely} satisfies the NSF conditions.
Nevertheless relaxed conditions,  allowing limited violation of the NSF and monotonicity conditions, still support our proof that neighbourhood search outperforms blind search.

Even when neighbourhood search has a higher probability of improvement than blind search, its expected amount of improvement may be less.  However local descent makes up for this and the paper shows how the full NSF property supports sufficient conditions for local descent to be expected to reach a target level of fitness faster than blind search.  

Local descent needs a starting solution, and the paper explores the cost - and benefit - of using blind search to find such a solution.
A hypothesis about the best level of fitness to accept as a starting solution found by blind search is proposed.

When a problem and neighbourhood exceed the minimum conditions we have set for beneficial local descent, the proof shows that the expected number of steps using local descent can be dramatically smaller than blind search. Benchmarks are constructed to show how this difference grows with steeper falls in the fitness probability, and stronger NSF conditions.

An extension of the proof for the case where the average NSF weight is less than one remains an open problem.
Also, 
the same intuitions hold for continuous problems 
whose fitness ranges are infinite.
Establishing these results for the continuous case is a challenge yet to be tackled.

Finally we are investigating an abstract model for local descent in which neighbourhood search can fail.

\section*{Acknowledgements}
This work was partly funded through the Monash University Outside Study Program.
My heartfelt thanks to Aldeida Aleti who came up with the name ``Neighbours Similar Fitness'' (NSF) and explored its potential with me for some two years. Peter Jeavons hosted me while I developed these ideas in Oxford, and made some pertinent critiques of earlier drafts of this work. David Cohen spent many hours with me and discovered the use of averages in the proof of the benefit of local descent.
Finally Markus Wagner came with fresh eyes, and helped clarify the relationship between NSF and function gradients.

\section{Declarations}
\begin{itemize}
    \item[Conflicts of Interest] There are no conflicts of interest involved in the production of
this manuscript.
    \item[] 
\end{itemize}

\bibliographystyle{plain}
\bibliography{aij}

\begin{appendix}

\newpage
\section{Terminology}
The following definitions are used in the paper.
\subsection{Neighbourhood search symbols}
The terms \emph{fitness} and \emph{cost} are synonymous in this paper.
\begin{itemize}
\item $f(s)$ is the fitness (or cost) of a solution $s$ in the search space $S$. 
\item  The finite fitness range of a problem class is $0 \ldots k_{max}$, where $0$ is the optimal fitness, and $k_{max}$ the worst.
\item $k_{mod}$ is the worst fitness below which $p(k)$ is monotonically decreasing.
\item $p(k)$ is the probability of a fitness $k \in 0 \ldots k_{max}$, in a problem class.
\item The probability that a neighbour of a fitness $k_1$ has fitness $k_2$ in a problem class is $pn(k_1,k_2)$.
\end{itemize}

\subsection{Defined terms}
\begin{flalign*}
p(k \pm \delta) =& p(k+\delta)+p(k-\delta)&\\
pn(k, k \pm \delta) =& pn(k,k+\delta)+ pn(k,k-\delta)&\\
k_{ge} =& k_{mod} / 2&\\
p^<(k) =& \sum_{\delta=1}^k p(k-\delta)&\\
pn^<(k) =& \sum_{\delta=1}^k pn(k,k-\delta)&\\
pbr^<(k) =& \sum_{\delta=1}^k p(k-\delta) \times r(k,\delta)&\\
p^>(k) =& \sum_{\delta=1}^{k_{max}-k} p(k+\delta)&\\
pbr^>(k) =& \sum_{\delta=1}^{k_{max}-k} p(k+\delta) \times r(k,\delta)&\\
imp(k) =& 1/p^<(k)&\\
imp_u(k) =& 1/pbr^<(k)&\\
blind =& 1/p(0)&
\end{flalign*}
\begin{flalign*}
steps(k) = &\left.
  \begin{cases}
  0 & \text{if} \  k = 0 \\
  imp(k) + \sum_{i<k} \frac{pn(k,i)}{ pn^<(k) } \times steps(i) & \text{if} \ k > 0
    \end{cases}
    \right\}\\
steps_u(k) = & \left.
  \begin{cases}
  0 & \text{if} \  k = 0 \\
  imp_u(k) + \sum_{i<k} \frac{r(k,i) \times p}{ pbr^<(k)} \times steps_u(i) & \text{if} \ k > 0
    \end{cases}
    \right\}\\
     blsteps(k) = & \ (p(0) \times 1) + \sum_{i=1}^{k-1} p(i) \times (1+steps(i))& \\
     & + (1-\sum_{i=0}^{k-1} p(i)) \times (1 + blsteps(k))& \\
\end{flalign*}

\section{Proof that Neighbourhood Search Is Beneficial}
\label{sec:improve-proof}
Lemma \ref{lemma:improvement}:\\
Assuming
\begin{alignat*}{3}
&k \leq&&  k_{ge} &(GE)\\
&\forall \delta_1 &&< \delta_2 : r(k,\delta1) \geq r(k,\delta2) &\ \ \ \ (NSF) 
\end{alignat*}
it follows that
$$\frac{pbr^<(k)}{pbr^>(k)} \geq \frac{p^<(k}{p^>(k)}$$
\vspace{0.5cm}

\noindent We write $\bar{r}(k) = (1/k) \times \sum_{\delta=1}^k r(k,\delta)$.

\begin{proof} \label{proof:impr}
We first assert two lemmas, proven below.
\begin{equation*}
k \leq k_{ge} \rightarrow \sum_{i=1}^k r(k,i) \times p(k-i) \geq \bar{r}(k) \times \sum_{i=1}^k p(k-i) \ \ {\rm Lemma \ }\ref{lemma:decdec}
\end{equation*}
\begin{equation*}
k \leq k_{ge} \rightarrow \sum_{i=1}^k r(k,i) \times p(k+i) \leq \bar{r}(k) \times \sum_{i>0} p(k+i)\ {\rm Lemma \ } \ref{lemma:allincdec}
\end{equation*}

Using these lemmas the proof is straightforward:
\begin{alignat*}{1}
\ \ \ &\frac{pbr^<(k)}{pbr^>(k)}\\
 \ \ = &\frac{\sum_{i=1}^k p(k-i) \times r(k,i)}{\sum_{i=1}^{k_{max}-k} p(k+i) \times r(k,i)}\\
 \ \ = &\frac{\sum_{i=1}^k p(k-i) \times r(k,i)}{\sum_{i=1}^k p(k+i) \times r(k,i) + \sum_{i=k+1}^{k_{max}-k} p(k+i) \times r(k,i)}\\
 \ \  \geq & \frac{\sum_{i=1}^k p(k-i) \times \bar{r}(k)}{\sum_{i=1}^k p(k+i) \times \bar{r}(k) +  \sum_{i=k+1}^{k_{max}-k} p(k+i) \times \bar{r}(k)} \ \ \rm{by \ lemmas \ \ref{lemma:decdec}, \ and \ \ref{lemma:allincdec}}\\
 \ \  = & \frac{\sum_{i=1}^k p(k-i)}{\sum_{i=1}^{k_{max}-k} p(k+i)}\\
 \ \ = & \frac{p^<(k}{p^>(k)}\\
\end{alignat*}
\end{proof}

Lemma \ref{lemma:pkk}:\\
Assuming
\begin{alignat*}{3}
&pn(k,k) \leq&&  p(k) &(A1)\\
&\frac{pbr^<(k)}{pbr^>(k)} \geq&& \frac{p^<(k}{p^>(k)} &\ \ \ \ (A2) 
\end{alignat*}
it follows that
$$pbr^<(k) \leq p^<(k) $$
\begin{proof}
Since\\
\ \ $pn(k,k)+pbr^<(k)+pbr^>(k) = 1 = p(k)+p^<(k)+p^>(k)$\\
\vspace{0.5ex}\\
$\therefore pbr^<(k) +  pbr^>(k)  = p^<(k) + p^>(k) + (p(k)-pn(k,k))$\\
\vspace{0.5ex}\\
$\therefore pbr^<(k) +  pbr^>(k)  \geq p^<(k) + p^>(k)$ by (A1)\\
\vspace{0.5ex}\\
$\therefore pbr^<(k) \times (1 +  \frac{pbr^>(k)}{pbr^<(k)} ) \geq p^<(k) \times (1 + \frac{p^>(k)}{p^<(k)})$ \\
\vspace{0.5ex}\\
$\therefore pbr^<(k)  \geq p^<(k)$ by (A2)
\end{proof}

\begin{lemma}
\label{lemma:decdec}
Assuming:
\begin{alignat*}{3}
&k \leq&&  k_{ge} &(GE)\\
&\forall \delta_1 &&< \delta_2 : r(k,\delta1) \geq r(k,\delta2) &\ \ \ \ (NSF) 
\end{alignat*}
it follows that
$$pbr^<(k) \geq \bar{r}(k) \times p^<(k)$$
\end{lemma}
\begin{proof}
For brevity, we write $p_\delta$ for $p(k-\delta)$, $r_\delta$ for $r(k,\delta)$ and $\bar{r}$ for $\bar{r}(k)$.\\
Since $r_\delta$ is decreasing with increasing $\delta$, we set $x$ to be the largest index for which $r_\delta \geq \bar{r}$.\\
By the definition of $k_{ge}$, since $\forall \delta : p_\delta \leq p(k) \leq k_{ge}$, then $p_\delta$ is also decreasing with increasing $\delta$ so\\
\vspace{0.5ex}\\
$\delta \leq x \rightarrow (p_\delta - p_x) \times (r_\delta - \bar{r}) \geq 0$ because $p_\delta-p_x$ and $r_\delta - \bar{r}$ are both positive\\
\vspace{0.5ex}\\
$\delta > x \rightarrow (p_\delta - p_x) \times (r_\delta - \bar{r}) \geq 0$ because $p_\delta-p_x$ and $r_\delta - \bar{r}$ are both negative\\
\vspace{0.5ex}\\
Therefore:
\begin{equation*}
\begin{split}
    & \sum_{\delta=1}^k p_\delta \times r_\delta \\
  = & \sum_{\delta=1}^x (p_\delta - p_x) \times (r_\delta - \bar{r}) + \sum_{\delta=x+1}^k (p_\delta - p_x) \times (r_\delta - \bar{r})\\
    & + \sum_{\delta=1}^k p_\delta \times \bar{r} + \sum_{\delta=1}^k p_x \times r_\delta + \sum_{\delta=1}^k p_x \times \bar{r} \\
  > & \sum_{\delta=1}^k p_\delta \times \bar{r}
    \end{split}
\end{equation*}
\end{proof}

\begin{lemma}
Assuming:
\begin{alignat*}{3}
&k \leq&&  k_{ge} &(GE)\\
&\forall \delta_1 &&< \delta_2 : r(k,\delta1) \geq r(k,\delta2) &\ \ \ \ (NSF) 
\end{alignat*}
it follows that
$$pbr^>(k) \leq \bar{r} \times p^>(k)$$
\label{lemma:allincdec}
\end{lemma}
Lemma \ref{lemma:allincdec} is a direct consequences of the following two lemmas: \ref{lemma:incdec} and \ref{lemma:bargtm}.
The definition and proof of lemma \ref{lemma:incdec} and lemma \ref{lemma:bargtm} follow.

\begin{lemma}
\label{lemma:incdec}
Assuming:
\begin{alignat*}{3}
&k \leq&&  k_{ge} &(GE)\\
&\forall \delta_1 &&< \delta_2 : r(k,\delta1) \geq r(k,\delta2) &\ \ \ \ (NSF) 
\end{alignat*}
it follows that
$$ \sum_{\delta=1}^k p(k+\delta) \times r(k,\delta)  \leq \bar{r}(k) \times \sum_{\delta=1}^k p(k+\delta) $$
\end{lemma}
\begin{proof}
For brevity, we write $p_\delta$ for $p(k+\delta)$, $r_\delta$ for $r(k,\delta)$ and $\bar{r}$ for $\bar{r}(k)$.\\
By the definition of $k_{ge}=k_{mod}/2$, since $\forall \delta : p_\delta \leq p(k_{mod})$, then $p_\delta$ is increasing with increasing $\delta$ so\\
\vspace{0.5ex}\\
$\delta \leq x \rightarrow (p_\delta - p_x) \times (r_\delta - \bar{r}) \leq 0$ because only $p_\delta-p_x$ is negative\\
\vspace{0.5ex}\\
$\delta > x \rightarrow (p_\delta - p_x) \times (r_\delta - \bar{r}) \geq 0$ because only $r_\delta - \bar{r}$ is negative\\
\vspace{0.5ex}\\
Therefore:
\begin{equation*}
\begin{split}
    & \sum_{\delta=1}^k p_\delta \times r_\delta \\
  = & \sum_{\delta=1}^x (p_\delta - p_x) \times (r_\delta - \bar{r}) + \sum_{\delta=x+1}^k (p_\delta - p_x) \times (r_\delta - \bar{r})\\
    & + \sum_{\delta=1}^k p_\delta \times \bar{r} + \sum_{\delta=1}^k p_x \times r_\delta + \sum_{\delta=1}^k p_x \times \bar{r} \\
  \leq & \sum_{\delta=1}^k p_\delta \times \bar{r}
    \end{split}
\end{equation*}

\end{proof}

\begin{lemma}
\label{lemma:bargtm}
Assuming:
\begin{alignat*}{3}
&\forall \delta_1 &&< \delta_2 : r(k,\delta1) \geq r(k,\delta2) &\ \ \ \ (NSF) 
\end{alignat*}
it follows that
$$\sum_{\delta>k}r(k,\delta) \times p(k+\delta) \leq \bar{r}(k) \times \sum_{\delta>k} p(k+\delta)$$
\end{lemma}
\begin{proof}
In this case $p(k+\delta)$ is not necessarily increasing with $\delta$, so lemma \ref{lemma:incdec} cannot apply.
However, since $r(k,\delta)$ is decreasing,\\ 
$\bar{r}(k) \geq r(k,k)$, and $\delta>k \rightarrow r(k,k) \geq r(k,\delta)$.
Therefore $$\sum_{\delta>k} p(k+\delta) \times r(k,\delta)  \leq \sum_{\delta>k} p(k+\delta) \times r(k,k) \leq \bar{r}(k) \times \sum_{\delta>k} p(k+\delta)$$

\begin{lemma}
Strict improvement:
if $k_{max}>k_{mod}$ and $r(k,1)>1$ then the consequence of theorem \ref{thm:improvement} is strict.\\
Assuming:
\begin{flalign*}
&k_{max} > k_{mod}&  &\\
&r(k,1) > 1&  &\\
&k \leq  k_{ge}& &(GE)\\
&\forall \delta_1 <\delta_2 : r(k,\delta1) \geq r(k,\delta2)& &(NSF(k)) \\
&\forall \delta : pn(k,k-\delta) \geq r(k,\delta) \times p(k-\delta)& &\\
&p(k) \geq pn(k,k) & &
\end{flalign*}
it follows that
$$pn^<(k) > p^<(k)$$
\label{lemma:strict}
\end{lemma}
We prove that if $r(k,1)>1$ the consequence of lemma \ref{lemma:bargtm} is strict, and the result follows.
If $r(k,1)>1$, then 
\begin{itemize}
    \item[either] $r(k,1) > \bar{r}(k) > r(k,k)$ 
     \item[or] $r(k,1)=r(k,k)=\bar{r}(k) > 1$.
\end{itemize}
In the first case:
$$\sum_{\delta=1}^k p(k+\delta) \times r(k,\delta) < \sum_{\delta>k} p(k+\delta) \times r(k,k) \leq \bar{r}(k) \times \sum_{\delta>k} p(k+\delta)$$ 
so the inequality is strict.
In the second case: 
$$\sum_{\delta \leq k} p(k \pm \delta) \times r(k,\delta) = \sum_{\delta \leq k} p(k \pm \delta) \times \bar{r}(k) > \sum_{\delta \leq k} p(k \pm \delta)$$
We know that:
$$\sum_{\delta>k} p(k+\delta) \times r(k,\delta) + \sum_{\delta \leq k} p(k \pm \delta) \times r(k,\delta) = \sum_{\delta>k} p(k+\delta) + \sum_{\delta \leq k} p(k \pm \delta)$$
Therefore 
$$\sum_{\delta>k} p(k+\delta) \times r(k,\delta) + \sum_{\delta \leq k} p(k \pm \delta) < \sum_{\delta>k} p(k+\delta) + \sum_{\delta \leq k} p(k \pm \delta)$$
So we conclude that:\\
$$\sum_{\delta>k} p(k+\delta) \times r(k,\delta) < \sum_{\delta>k} p(k+\delta)$$
In this case the inequality is again strict.
\end{proof}

\section{Proof of the benefit of local descent}
\label{sec:localdescentproof}

\subsection{Proof for uniform fitness probabilities}
Lemma \ref{lemma:localdescent} states that for any $k$ where \\
$\bar{r}(k) \geq 1$\\
$\forall \delta_1 \leq k_1, \delta_1 \leq \delta_2, \delta_2 \leq k_2, k_2 \leq k :  r(k_1,\delta_1) \geq r(k_2,\delta_2)$ \ (full NSF)\\
it \ follows \ that 
\begin{flalign*}
\ \ \ steps_u(k)& \leq blind&
\end{flalign*}
Writing $imp_u(k) = \frac{1}{p \times \sum_{\delta=1}^k r(k,\delta)}$

\[
steps_u(k) = \left.
  \begin{cases}
  0 & \text{if} \  k = 0 \\
  imp_u(k) + \sum_{i<k} \frac{p \times r(k,i)}{ p \times \sum_{j<k} r(k,j)} \times steps_u(i) & \text{if} \ k > 0
    \end{cases}
    \right\}
\]
The proof uses two sublemmas: lemma \ref{lemma:gandt} and lemma \ref{lemma:main}.
\begin{lemma}
Steps Upper Bound\\ 
Assuming $\bar{r}(k) \geq 1$
then
 $imp_u(k) \times k  \leq blind$
\label{lemma:gandt}
\end{lemma}
\begin{proof}
By definition:
\begin{flalign*}
blind &= 1/p&
\end{flalign*}
also
\begin{flalign*}
imp_u(k) &= 1 / (p \times \sum_{0 \leq i<k} r(k,i))&
\end{flalign*}
Therefore:
    \begin{flalign*}
       blind & = 1/p & \\
       & = imp_u(k) \times \sum_{0 \leq i<k} r(k,i)& \\
       & = imp_u(k) \times \bar{r}(k) \times k&\\
       & \geq imp_u(k) \times k &\\
    \end{flalign*}
\end{proof}

\begin{lemma}[Fixed Count Steps]
Assuming $\bar{r}(k) \geq 1$\\ 
$\forall \delta_1 \leq k_1, \delta_1 \leq \delta_2, \delta_2 \leq k_2, k_2 \leq k :  r(k_1,\delta_1) \geq r(k_2,\delta_2)$ \ (full NSF)\\
then
\begin{equation*}
 steps_u(k) \leq k \times imp_u(k) 
\end{equation*}
\label{lemma:main}
\end{lemma}
\begin{proof} 
The proof is by induction, using the definition of $steps(k)$ to make the inductive step.

Firstly, we establish the base case $k=1$.
In this case lemma (\ref{lemma:main}) holds because 
\begin{equation*}
    \begin{split}
steps_u(1) &= imp_u(1) + \frac{ r(1,1) \times steps_u(0)}{r(1,1)} \\
                 &= imp_u(1)
    \end{split}
\end{equation*}

For induction we assume lemma (\ref{lemma:main}) holds for all $steps_u(i): i < k$.
To prove it holds for $k$, we substitute 
$(i) \times imp_u(i)$ for $steps_u(i)$ in the definition of $steps_u(k)$. Since $steps_u(i)$ occurs positively in this definition, our inductive assumption ensures this yields an expression greater than $steps_u(k)$.
 
\begin{equation*}
\begin{split}
steps_u(k) & = imp_u(k)\times ( 1 + \sum_{i=1}^{k-1} (p \times r(k,k-i) \times steps_u(i) ))\\
& \leq imp_u(k) \times ( 1 + \sum_{i=1}^{k-1} (p \times r(k,k-i) \times imp_u(i) \times i))\\
& = imp_u(k) \times ( 1 + \sum_{i=1}^{k-1} (p \times r(k,k-i) \times \frac{(i)}{\sum_{\delta=1}^{i} r(i,\delta) \times p}))\\
& = imp_u(k) \times ( 1 + \sum_{i=1}^{k-1} (r(k,k-i) \times \frac{(i)}{\sum_{\delta=1}^{i} r(i,\delta)}))\\
&\rm{by \ \ lemma \ \ (\ref{lemma:rmi}), \ below}\\
& \leq imp_u(k) \times ( 1 + \sum_{i=1}^{k-1} r(k,k-i) \times \frac{((k-1)-k_{opt})}{\sum_{\delta=1}^{(k-1)-k_{opt}} r(k,\delta)})\\
& = imp_u(k) \times ( 1 + \sum_{i=1}^{(k-1)-k_{opt}} r(k,i) \times \frac{((k-1)-k_{opt})}{\sum_{\delta=1}^{(k-1)-k_{opt}} r(k,\delta)})\\
& = imp_u(k) \times ( 1 + (k-1) - k_{opt} ) \\
& = (k-k_{opt}) \times imp_u(k)
\end{split}
\end{equation*}
\end{proof}

\begin{lemma}
\label{lemma:rmi}
$$\forall i<(k-1) : 
\frac{((k-1)-k_{opt})}{\sum_{\delta=1}^{(k-1)-k_{opt}} r(k,\delta)} \geq \frac{(i)}{\sum_{\delta=1}^{ i} r(i,\delta)}$$
\end{lemma}

\begin{proof}
By $NSF(k-1)$, $r(k-1,\delta)$ is decreasing with increasing $\delta$, and accordingly the average
$\sum_{\delta=1}^n r(k-1,\delta)/n$ is also decreasing with increasing $n$.

Consequently since $(k-1) > i $
$$\frac{\sum_{\delta=1}^{(k-1)-k_{opt}} r(k-1,\delta)}{(k-1)-k_{opt}} \leq \frac{\sum_{\delta=1}^{i} r(k-1,\delta)}{i}$$
By full NSF, if $k \geq i$ then $r(k,\delta) \leq r(i,\delta)$, and accordingly 
$$\frac{\sum_{\delta=1}^{i} r(k-1,\delta)}{i} \leq \frac{\sum_{\delta=1}^{i} r(i,\delta)}{i}$$
Therefore 
$$\frac{(k-1)-k_{opt}}{\sum_{\delta=1}^{(k-1)-k_{opt}} r(k-1,\delta)} \geq \frac{(i)}{\sum_{\delta=1}^{ i} r(i,\delta)}$$
and finally, using full NSF again, $r(k-1,\delta) \geq r(k,\delta)$ and therefore:\\
$$\frac{(k-1)-k_{opt}}{\sum_{\delta=1}^{(k-1)-k_{opt}} r(k,\delta)} \geq \frac{(i)}{\sum_{\delta=1}^{ i} r(i,\delta)}$$
\end{proof}
Lemma \ref{lemma:localdescent} is an immediate consequence of lemma \ref{lemma:gandt} and lemma \ref{lemma:main}.

\subsection{Increasing the number of points}
We now prove that, under the assumptions of lemma \ref{lemma:localdescent} local search is beneficial for any problem class $s$ where $\forall k2 \leq k: p(k2) \geq p(0)=p$.
Let us write $steps_s(k)$ for the value of $steps(k)$ in this problem class.

Firstly assuming $pn$ is boosting, lemma \ref{lemma:gandt} implies that
$k \times imp(k) \leq blind$.

\begin{lemma}[Reduced Steps]
Assume $\bar{r}(k) \geq 1$\\
$\forall i \leq k : p(i) \geq p(0) = p$\\
$\forall \delta_1 \leq k_1, \delta_1 \leq \delta_2, \delta_2 \leq k_2, k_2 \leq k :  r(k_1,\delta_1) \geq r(k_2,\delta_2)$ \ (full NSF)\\
$\forall k_2 \leq k, \delta : pn(k_2,k_2-\delta) \geq r(k_2,\delta) \times p(k_2-\delta)$ (normal)\\ 
it follows that\\
$steps_s(k) \leq steps_u(k)$
\end{lemma}

\begin{proof}
The proof is by induction on $k$.
For the base case, $k=1$, and
$$steps_s(k) = 1/pn(1,0) \leq 1/(r(1,1) \times p(0)) = steps_u(k)$$

For the inductive step, we assume $\forall i \leq k : steps_s(i) \leq steps_u(i)$\\
We also use lemma \ref{lemma:monotonesteps} whose conclusion is
$\forall k_1 \leq k_2 \leq k_{ge} : steps_u(k_1) \leq steps_u(k_2)$

For readability we write:\\
$np$ is the probability of \emph{not} improving:\\
$np_u(k) = (1- pbr^<(k)) = 1 - \sum_{\delta=1}^k r(k,\delta) \times p$.\\
$np_s(k) = (1- pn^<(k)) = 1 - \sum_{\delta=1}^k pn(k,k-\delta)$\\
Note that:
\begin{flalign}
\label{eq:npminusns}
\notag np_u(k)-np_s(k) & = \sum_{\delta=1}^k pn(k,k-\delta) - (r(k,\delta) \times p) &\\
& \geq \sum_{\delta=1}^k pn(k,k-\delta) - (r(k,\delta) \times p(k-\delta))&
\end{flalign}

Suppose, for contradiction, that $steps_s(k) > steps_u(k)$.\\
Then
\begin{flalign*}
    steps_s(k) 
    =& 1 + np_s(k) \times steps_s(k) + \sum_{\delta=1}^k pn(k,k-\delta) \times steps_s(k-\delta)&\\
    \leq& 1 + np_s(k) \times steps_s(k) + \sum_{\delta=1}^k pn(k,k-\delta) \times steps_u(k-\delta)&\\
{\rm (because}& \ \rm{by \ inductive \ assumption \ } steps_u(k-\delta) > steps_s(k-\delta) \ )&\\
    =& 1 + np_s(k) \times steps_s(k)&\\
    &+ \sum_{\delta=1}^k (pn(k,k-\delta)-r(k,\delta) \times p(k-\delta)) \times steps_u(k-\delta)&\\
    &+ \sum_{\delta=1}^k r(k,\delta) \times p(k-\delta) \times steps_u(k-\delta)&\\
    \leq& 1 + np_s(k) \times steps_s(k) &\\
    &+ \sum_{\delta=1}^k (pn(k,k-\delta)-r(k,\delta) \times p(k-\delta)) \times steps_u(k) &\\
    &+ \sum_{\delta=1}^k r(k,\delta) \times p(k-\delta) \times steps_u(k-\delta)&\\
    {\rm ( because \ by}& \ \rm{lemma \ \ref{lemma:monotonesteps} \ } \forall \delta > 0 : steps_u(k) > steps_u(k-\delta) \ {\rm and}&\\
    {\rm by \ balanced}& \ \rm{or \ boosting \ } pn(k,k-\delta)-r(k,\delta) \times p(k-\delta) \geq 0 )&
\end{flalign*}
Therefore
\begin{flalign*}
    steps_s(k) \leq & 1 + np_s(k) \times steps_s(k) &\\
    &+ (np_u(k)-np_s(k)) \times steps_u(k) &\\
    &+ \sum_{\delta=1}^k r(k,\delta) \times p(k-\delta) \times steps_u(k-\delta) &\\
    &({\rm by \ equation \ \ref{eq:npminusns}}) &\\
    = & 1 + np_s(k) \times steps_s(k) - np_s(k) \times steps_u(k)&\\
    &+ np_u(k) \times steps_u(k) &\\
    &+ \sum_{\delta=1}^k r(k,\delta) \times p(k-\delta) \times steps_u(k-\delta) &\\
    \leq & 1 + np_u(k) \times steps_u(k) + \sum_{\delta=1}^k r(k,\delta) \times p(k-\delta) \times steps_u(k-\delta)&\\
    & {\rm (using \ assumption \ } steps_s(k) > steps_u(k) \ )&\\
    = & steps_u(k)&\\
\end{flalign*}
Since $steps_s(k) > steps_u(k) \rightarrow steps_s(k) \leq steps_u(k)$ it follow by contradiction that $steps_s(k) \leq steps_u(k)$
\end{proof}

\begin{lemma}[Monotonicity of steps]
\label{lemma:monotonesteps}
 Assuming \\
$\forall \delta_1 \leq k_1, \delta_1 \leq \delta_2, \delta_2 \leq k_2, k_2 \leq k :  r(k_1,\delta_1) \geq r(k_2,\delta_2)$ \ (full NSF)\\
then
\begin{flalign*}
\ \ \ \forall k_1 \leq k_2 \leq k &: steps_u(k_1) \leq steps_u(k_2)&
\end{flalign*}
\end{lemma}
The proof is by induction on $k$.
The base case $steps_u(1) \geq steps_u(0)$ is immediate because $steps_u(1)$ is non-negative and $steps_u(0)=0$.

For the inductive case we can assume that
\begin{flalign}
\forall i<k-1 : steps_u(k-1) \geq steps_u(i)
\label{eq:indass}
\end{flalign}
\vspace{0.5cm}


\noindent Since $\forall k, \delta: r(k,\delta) \leq r(k-1,\delta-1)$ we can write $r(k-1,\delta-1) = \epsilon_{\delta} + r(k,\delta) $, where $\epsilon_{\delta} \geq 0$.
Consequently
\begin{equation}
\sum_{\delta=2}^k r(k,\delta) \times p ) = \sum_{\delta=1}^{k-1}(r(k-1,\delta-1) - \epsilon_{\delta}) \times p)
\label{eq:skminus1}  
\end{equation}
The following equation is also used in the proof

\begin{equation}
steps_u(k) \times pbr^<(k) = 1 + \sum_{\delta=1}^{k-1} r(k,\delta) \times p \times steps_u(k-\delta)
\label{eq:stkminus1}  
\end{equation}

\noindent Assume, for contradiction, that:
\begin{flalign}
steps_u(k) & < steps_u(k-1)&
\label{eq:asssteps2}
\end{flalign}

\begin{proof} $$steps_u(k) \geq steps_u(k-1)$$

\begin{flalign*}
steps_u(k) = & 1 + (1-\sum_{\delta=1}^k r(k,\delta) \times p) \times steps_u(k)&\\
            & + r(k,1) \times p \times steps_u(k-1) + \sum_{\delta=2}^k r(k,\delta) \times p \times steps_u(k-\delta)&\\
(\rm{by} \ \ref{eq:asssteps2}) \ \  \geq & 1 + (1- \sum_{\delta=2}^k r(k,\delta) \times p) \times steps_u(k) +  \sum_{\delta=2}^k r(k,\delta) \times p \times steps_u(k-\delta)&\\
\therefore 
steps_u(k) \times & \sum_{\delta=2}^k r(k,\delta) \times p \geq
        1 + \sum_{\delta=2}^{k} r(k, \delta) \times p  \times steps_u(k-\delta)& \\
\therefore 
steps_u(k) \times & \sum_{\delta=1}^{k-1}(r(k-1,\delta-1) - \epsilon_{\delta}) \times p&\\
\geq & 1+
  \sum_{\delta=1}^{k-1}(r(k-1,\delta-1) - \epsilon_{\delta}) \times p \times steps_u((k-1)-\delta)&\\
(\rm{by} \ \ref{eq:indass}) \ \
\geq & 1 + \sum_{\delta=1}^{k-1} r(k-1,\delta-1) \times p  \times steps_u((k-1)-\delta) - \sum_{\delta=1}^{k-1}\epsilon_{\delta} \times p \times steps_u(k-1)&\\
(\rm{by} \ \ref{eq:stkminus1}) \ \ 
\geq & steps_u(k-1) \times pbr^<(k-1) - \sum_{\delta=1}^{k-1}\epsilon_{\delta} \times p \times steps_u(k-1)&\\
 = & steps_u(k-1) \times (pbr^<(k-1) - \sum_{\delta=1}^{k-1}\epsilon_{\delta} \times p)&\\
\therefore 
steps_u(k) \times & (pbr^<(k-1) - \sum_{\delta=1}^{k-1}\epsilon_{\delta} \times p) \geq steps_u(k-1) \times (pbr^<(k-1) - \sum_{\delta=1}^{k-1}\epsilon_{\delta} \times p)&\\
\end{flalign*}
Which implies that $steps_u(k) \geq steps_u(k-1)$.

The assumption $steps_u(k)<steps_u(k-1)$ implies $steps_u(k)\geq steps_u(k-1)$, and we conclude by contradiction that $steps_u(k) \geq steps_u(k-1)$
\end{proof}

\newpage
\section{Proof that preceding local descent by a blind search cannot make local descent worse than blind search}
\label{sec:blsearch}

\begin{theorem}
\label{prf:blindthenlocal}
For any $k$ where \\
$\bar{r}(k) \geq 1$\\
$\forall \delta_1<\delta_2<k_2 \leq k : r(k_2,\delta_1) \geq r(k_2,\delta_2)$ \ (full NSF)\\
$\forall \delta \leq k_2 < k_1 \leq k : r(k_1,\delta) \leq r(k_2,\delta)$ (full NSF)\\
then\\
$$blsteps(k) \leq blind$$
\end{theorem}
\begin{proof}

\begin{flalign*}
     blsteps(k) & =  p(0)  + \sum_{i=1}^k p(i) \times (1+steps(i))
     + \sum_{i>k} p(i) \times (1 + blsteps(k)) &\\
     & = 1 + \sum_{i=1}^k p(i) \times steps(i) + \sum_{i>k} p(i) \times blsteps(k)&\\
    & = 1 + \sum_{i=1}^k p(i) \times steps(i) + (1- \sum_{i=0}^k p(i)) \times blsteps(k)&\\
     \therefore \  blsteps(k) & \times \sum_{i=0}^k p(i) 
     = 1 + \sum_{i=1}^k p(i) \times steps(i) &\\
     & \ \ \ \ \ \ \ \ \leq 1 + \sum_{i=1}^k p(i) \times blind &\\
\end{flalign*}
since $steps(i) \leq blind$.

The expected number of steps required for blind search is:
\begin{flalign*}
     blind &= p(0) +  (1-p(0)) \times (1+ blind) & \\
          &= 1 + \sum_{i=1}^k p(i) \times blind 
               + \sum_{i>k} p(i)) \times blind & \\
         &= 1 + \sum_{i=1}^k p(i) \times blind 
            + (1- \sum_{i=0}^k p(i)) \times blind &\\
\therefore 
 blind & \times \sum_{i=0}^k p(i) = 1 + \sum_{i=1}^k p(i) \times blind&  \\
\end{flalign*}
Therefore $blsteps(k) \leq blind$.
\end{proof}

\section{Relaxing any requirement can make neighbourhood search worse than blind search}
\label{sec:counterex}
\subsection{Requirements}
The properties that support a proof that neighbourhood is beneficial are as follows (see theorem \ref{thm:improvement}):
\begin{requirement}
\label{req:1}
\begin{equation}
\forall 0 \leq k1 < k2 \leq k_{mod} : p(k1) \leq p(k2)
\end{equation}
\end{requirement}
The proof establishes that neighbourhood search and local descent are beneficial from cost of $k_{ge}$ or better.

\begin{requirement}
\label{req:rkidecr}
\emph{NSF(k)}\\
The $NSF(k)$ property says that the NSF weight $r(k,\delta)$ increases as $\delta$ decreases: 
\begin{equation}
\forall k \leq k_{ge} : \delta_1 \leq \delta_2 \rightarrow r(k,\delta_1) \geq r(k,\delta_2)
\end{equation}
\end{requirement}
This simply says that neighbours tend to have similar cost - and the closer the cost the more likely a neighbour has it.

The next property is that neighbourhoods are not ``biased'' towards worse cost values.
\begin{requirement}
\label{req:3}
\emph{Normal}\\
The proportion of improving neighbours at a given cost distance $\delta$ is likely to be at least as great as the proportion in the search space as a whole, i.e.
\begin{equation}
\forall k,\delta {\rm \ where \ } p(k,k \pm \delta) > 0 : \frac{pn(k,k-\delta)}{pn(k,k \pm \delta)} \geq \frac{p(k-\delta)}{p(k \pm \delta)} 
\end{equation}
\end{requirement}
Clearly if most neighbours of any point have better cost, neighbourhood search is bound to be beneficial - but if it is computationally practical to construct such neighbourhoods, the same method can be used to find an optimal solution.

\begin{requirement}
\label{req:4}
\emph{Improvement condition}\\
$\forall k \leq k_{ge} : pn(k,k) \leq p(k)$\\
\end{requirement}

In the following subsections, for each requirement we present a problem class and neighbourhood which satisfies the other requirements but where neighbourhood search fails to be beneficial.  We thus show that all three requirements are needed.
Note that, given a starting cost of $k$,
\begin{itemize}
    \item the probability of finding an improving cost with blind search is\\ $p^<(k) = \sum_{i=0}^{k-1} p(i)$
    \item the probability of finding an improving cost with neighbourhood search is
    $pn^<(k) = \sum_{i=0}^{k-1} pn(k,i)$
\end{itemize}

\subsection{A problem where {\em p(i)} is not monotonically decreasing below the starting cost {\em k=25}}
In this problem $p(0)$ is large, but at all other cost levels uniformly small.
cost level $k>0$ is a neighbour of the two  cost levels $k+1$ and $k-1$ only.
This satisfies the boosting property \ref{req:3}.
The NSF property is strong - all neighbours have a fitness difference of $1$.\\
$k_{max}=100 ; k_{opt}=0 ;  \forall i > 0 : p(i)=1/100$\\
$r(25,1)= 100$ \\
$\forall \delta \neq 1 : r(25,\delta) = 0$

These neighbourhood weights satisfy both requirements $NSF(k)$ (\ref{req:rkidecr}) and $pn(k,k) \leq p(k)$
(\ref{req:4}).    

From cost $25$, the probability of finding an improving cost with blind search is $24+100/200 = 0.62$.
The probability of finding an improving neighbour, of cost $25$, with neighbourhood search is $0.5$.
Thus neighbourhood search is not beneficial.

\subsection{Neighbourhood weights which are not NSF}
In this problem the cost probability is uniform at all cost levels, so $k_{mod}=k_{max}$.
Each cost level only has neighbours at fitness difference $26$.
$pn(25,25)=0$ satisfying requirement $pn(k,k) \leq p(k)$ (\ref{req:4}).
For all $i \leq 25 : pn(25,25+i) = pn(25,25-i) = 0$ so the neighbourhood is normal(\ref{req:3}).
The details of the problem and NSF weights are as follows:\\
$k_{max}=100 ; k_{opt}=0 ; \forall i : p(i)=1/101$\\
$k_{ge}=50 ; r(25,26)= 101$\\
$\forall \delta \neq 26 : r(25,\delta) = 0$

From cost $25$, the probability of finding an improving cost with blind search is $25/101 = 0.248$.
The probability of finding an improving neighbour with neighbourhood search is $0$.
Thus neighbourhood search is not beneficial.

\subsection{Neighbourhoods that are not normal}
In this problem the cost probability is uniform at all cost levels, so $k_{mod}=k_{max}$.

cost $k$ is a neighbour of only one other cost: $k+1$.
These neighbourhood weights satisfy both requirements $NSF(k)$ (\ref{req:rkidecr}) and $pn(k,k)\leq p(k)$ (\ref{req:4}).
Neighbourhoods are not symmetric, so if $s1$ is a neighbour of $s2$ we cannot infer that $s2$ is a neighbour of $s1$.

The details of the problem and NSF weights are as follows:\\
$k_{max}=100 ; k_{opt}=0 ; \forall i : p(i)=1/101$\\
$k_{ge}=50 ; r(25,1)= 101; pn(25,26)=1$ \\
$\forall \delta \neq 1 : r(25,\delta) = 0$\\

From cost $25$, the probability of finding an improving cost with blind search is $25/101 = 0.248$.
The probability of finding an improving neighbour with neighbourhood search is $0$.
Thus neighbourhood search is not beneficial.

\subsection{Neighbourhood weights without the improvement condition}
In this problem, also, the cost probability is uniform at all cost levels,
so $k_{mod}=k_{max}$
In this case all neighbours have the same cost level,
which satisfies requirement $NSF(k)$ for every cost $k$ (\ref{req:rkidecr}), and normal (\ref{req:3}).
The details of the problem and NSF weights are as follows:\\
$k_{max}=100 ; k_{opt}=0 ; \forall i : p(i)=1/101$\\
$k_{ge}=50 ; r(25,0)= 101$\\
$\forall \delta>0 : r(25,\delta) = 0$

From cost $25$, the probability of finding an improving cost with blind search is $25/101 = 0.248$.
The probability of finding an improving neighbour with neighbourhood search is $0$.
Thus neighbourhood search is not beneficial.






\section{Calculating the problem class 2-SAT}
\label{app:2SAT}
The 2-SAT problem class is defined as follows:
\begin{itemize}
\item 100 clauses each with two variables
\item 50 distinct variables
\item Each variable occurs in 4 clauses
\item Minimise violated clauses
\end{itemize}
\noindent A point has cost $k$ if 
$k$ clauses are unsatisfied.\\
We find a neighbour of a point with cost $k$ by flipping a variable.\\
The (four) clauses in which the flipped variable occur
are unsatisfied with probability $k/100$, and
satisfied with probability $(100-k)/100$.\\
Flipping a variable in an unsatisfied clause always makes it
satisfied.\\
Flipping a variable in a satisfied 2-clause makes it unsatisfied
with probability $1/3$.

The cost change after flipping is $0$ iff
\begin{itemize}
\item  all four clauses in which the flipped variable appears are satisfied,
and each time the clause stays satisfied with probability $2/3$\\
or 
\item one of the clauses is satisfied and flipping the variable makes the
clause false; one of the clauses is unsatisfied and flipping the
variable makes it true (necessarily); and two of the
clauses are satisfied but flipping the variable does not make it
false\\ 
(there are 12 ways of this occurring in four clauses)\\
or 
\item two of the clauses are unsatisfied (and necessarily become
true); and two are satisfied and flipping the variable makes
both false\\ 
(there are 6 ways of this occurring in four clauses)
\end{itemize}
Thus, writing $p_k = k/100$, $npf_{k} = (1-p_k)* 1/3$ and $npt_{k} = (1-p_k) \times 2/3$:
\begin{flalign*}
pn(k,k) = \ & npt_{k}^4 &\\
&+ 12 \times npf_{k} \times p_k \times npt_{k}^2 &\\
&+ 6 \times p_k^2 \times npf_{k}^2
\end{flalign*}
The number of unsatisfied clauses decreases by one if:
\begin{itemize}
    \item one clause becomes true, and the others are unchanged\\
    (there are 4 ways of this occurring)\\
    or
    \item two clauses become true, one becomes false and one is unchanged\\
    (there are 12 ways of this occurring)
\end{itemize}
Thus:
\begin{flalign*}
pn(k,k-1) = \ & 4 \times p_k \times npt_k^3 &\\
& + 12 \times p_k^2 \times npf_k \times npt_k 
\end{flalign*}
The number increases by one if:
\begin{itemize}
    \item one clause becomes false, and the others are unchanged\\
    (there are 4 ways of this occurring)\\
    or
    \item two clauses become false, one becomes true and one is unchanged\\
    (there are 12 ways of this occurring)\\
\end{itemize}
Thus:
\begin{flalign*}
pn(k,k+1) = \ & 4 \times npf_k \times npt_k^3 &\\
& + 12 \times npf_k^2 \times p_k \times npt_k 
\end{flalign*}
These are the two different ways of improving the cost by 2:\\
\begin{flalign*}
pn(k,k-2) = \ & 6 \times p_k^2 \times npt_k^2 &\\
& 4 \times p_k^3 \times npf_k 
\end{flalign*}
and two ways of worsen the cost by 2:
\begin{flalign*}
pn(k,k+2) = \ & 6 \times npf_k^2 \times npt_k^2 &\\
& 4 \times npf_k^3 \times p_k 
\end{flalign*}
There is only one way to improve cost by 3:\\
$pn(k,k-3) = 4 \times p_k^3 \times npt_k$\\
and one way to worsen the cost by 3:\\
$pn(k,k+3) = 4 \times npf_k^3 \times p_k$\\
There is one way to improve the cost by 4:\\
$pn(k,k-4) = p_k^4$\\
and one way to make it worse by 4:\\
$pn(k,k+4) = npf_k^4$\\
For $\delta>4$:\\
$pn(k,k-\delta) = pn(k,k+\delta) = 0$

\section{ECLiPSe program for computing steps}
 \label{sec:eclipseprog}
 The program is invoked for the example in the text as follows:
 \begin{verbatim}
 ?- Steps is goal(9, 
                  [1, 5, 25, 125, 625, 3125, 15625, 78125, 390625], 
                  500000, 
                  [4, 3.5, 3, 0.5]).
Steps = 128801.72941244145__128801.72941244178
Yes (0.00s cpu)    
 \end{verbatim}
 
 ECLiPSe is free and downloadable at \verb0www.https://www.eclipseclp.org/0
 
 The program implementing $goal$ is:
 
 \begin{Verbatim}[fontsize=\small]
?- lib(ic).

/* Size is the number of fitness levels between the optimum and modal fitness
Sols is a list of counts of solutions at each of these fitness levels
Total is the total number of solutions in the search space
D is a variable for the list of values of r(_,i): i >=1
Steps is the expected number of steps to reach the optimum fitness

Example:
goal(9,[1,5,25,125,625,3125,15625,78125,390625],600000,[2,1.5,1.2,1],Steps).
*/


goal(Size,Sols,Total,NSFWeights,Steps) :-
    init(Size,Sols,Total,AS),
    % K is the current fitness
    K is Size div 2,
    % Construct and constrain the array AD of increased probabilities
    % of neighbours at this distance
    cons_ad(NSFWeights,K,Total,AS,AD),
    % Construct and constrain the array of probabilities of an
    % improving neighbour
    cons_aimp(K,Total,AS,AD,Aimp),
    % Constrain the number of steps to reach the optimum
    cons_steps(K,Total,AS,AD,Aimp,Steps),
    
    % Constrain this number to be worse than random search
    % When uncommented, even if NSFWeights is unknown (a variable)
    % the program reports that all the NSFWeights must be close to 1
%   AS[1] * Steps $>= Total,
    
    % Apply interval shaving to prune infeasible ranges for the variables
    squash([Steps|NSFWeights],0.00001,lin).

% Construct an array StepsArray[K] is the number o steps to reach the
% optimum from fitness ditance K
cons_steps(K,Total,AS,AD,Aimp,Steps) :-
    functor(StepsArray,s,K),
    Steps $= StepsArray[K],
    % The expected number of steps from Opt-1
    StepsArray[1]*Aimp[1] $= 1.0,
    % Impose the function computing the number of steps as a constraint
    (for(I,2,K), param(AS,Total,AD,Aimp,StepsArray) do 
         (StepsArray[I]*Aimp[I]*Total $= 
          Total + josum(J,1,I-1,AD[J]*StepsArray[I-J]*AS[I+1-J]))
    ).

% Constrain the possible values for the array of D_i
cons_ad(D,K,Total,AS,AD) :-
    length(D,K),
    D $:: 0.01..100.0,
    AD =.. [ip|D],
    % NSF constraint
    (for(I,1,K-1), param(AD) do AD[I+1] $=< AD[I]),
    % Constraint that the sum of all the probabilities across the
    % whole solution space is 1
    Total-sum(AS) $= LHS,
    rhs(Total,K,AD,AS,RHS),
    (var(RHS) -> true ; check(RHS >= 0,total_too_low(Total))),
    RHS $>=0,
    % Give all the distant points (LHS) the same probability PP
    PP * LHS $= RHS,
    (LHS =:= 0 -> PP is 0.0 ; true) ,
    % Ensures the probabilities of distant fitnesses is worse than the
    % lowest r-value given in the input list AD
    test_total(PP,AD,K).

test_total(PP,AD,K) :-
    RK is AD[K], 
    (var(RK) -> true ; check(PP=<RK,last_NSFWeight_too_low(RK))),
    PP $=< AD[K].

% Add all the ct(I)*pn(Mode, +/- I) and make sure this is less than
% the total number of points in the search space
rhs(Total,K,AD,AS,RHS) :-
    RHS $= (Total - 
            (josum(I,1,K,AD[I]*AS[K+1-I]) +
             josum(I,1,K,AD[I]*AS[K+1+I]) +
             % r(k,0) = 1
             1*AS[K+1])
           ).        

% Constraints on probability of improving
cons_aimp(K,Total,AS,AD,Aimp) :-
    functor(Aimp,a,K),
    % The function defining the probability of improving as a constraint
    (for(I,1,K), param(AS,Total,AD,Aimp) do
        Aimp[I]*Total $= josum(J,1,I,AD[J]*AS[I+1-J])
    ).

init(Size,Sols,Total,AS) :-
    check(nonvar(Size),error_var(Size)),
    check(odd(Size), error_even(Size)),
    Total #=< 10000000,
    check(length(Sols, Size),error_sols_length(Sols)),
    (foreach(S,Sols) do S #> 0),
    check(sum(Sols) $=< Total, total_too_low(Total)),
    AS =.. [[]|Sols].


%  Utilities

odd(Int) :- Int mod 2 =:= 1.

check(Goal,Error) :- ((not Goal) -> writeln(Error), abort ; true).

% sum(Var in LB..UB)(Expr(Var)) = Sum
josum(Var,LB,UB,Expr,Sum) :-
        Min is eval(LB),
        Max is eval(UB),
    ( for(K,Min,Max), fromto(0,This,Next,Sum), param(Var,Expr) do
        copy_term_vars(Var, Var-Expr, K-KExpr),
	Next $= This+eval(KExpr)
    ).

	 
      

 \end{Verbatim}

\end{appendix}
\end{document}